\PassOptionsToPackage{table}{xcolor}
\documentclass{article}
\usepackage{arxiv}
\usepackage[utf8]{inputenc}
\usepackage[T1]{fontenc}

\usepackage{amsmath,amsfonts,bm}

\def\eqref#1{equation~\ref{#1}}

\def\1{\bm{1}}

\DeclareMathAlphabet{\mathsfit}{\encodingdefault}{\sfdefault}{m}{sl}
\SetMathAlphabet{\mathsfit}{bold}{\encodingdefault}{\sfdefault}{bx}{n}

\usepackage{url}
\usepackage{booktabs}
\usepackage{multirow}
\usepackage{amsmath,amsfonts,amssymb,amsthm}
\usepackage{nicefrac}
\usepackage{microtype}
\usepackage{graphicx}
\usepackage{xcolor}
\usepackage{enumitem}
\usepackage[round]{natbib}
\usepackage{algorithm}
\usepackage{algpseudocode}
\usepackage[skip=2pt]{caption}

\usepackage{hyperref}

\graphicspath{{figs/}}

\definecolor{refrow}{gray}{0.925}
\definecolor{trainedclr}{HTML}{2E6FB7}
\definecolor{frozenclr}{HTML}{99330B}
\newcommand{\trainpar}[1]{\textcolor{trainedclr}{#1}}
\newcommand{\freezepar}[1]{\textcolor{frozenclr}{#1}}

\newcommand{\factor}[2]{\underbrace{#1\vphantom{M_k^{\top}}}_{\text{#2}}}
\newtheorem{proposition}{Proposition}

\makeatletter
\g@addto@macro\normalsize{%
  \setlength\abovedisplayskip{4pt plus 2pt minus 2pt}%
  \setlength\belowdisplayskip{4pt plus 2pt minus 2pt}%
  \setlength\abovedisplayshortskip{0pt plus 2pt}%
  \setlength\belowdisplayshortskip{2pt plus 2pt minus 1pt}}
\makeatother

\title{SOLO: Pretraining Billion-Parameter Language Models with Shared-Output Local Learning}
\renewcommand{\shorttitle}{SOLO: Shared-Output Local Learning}

\author{%
  Bojian Yin\thanks{Corresponding author.} \qquad
  Shurong Wang \qquad
  Yuqi Pan \qquad
  Guoqi Li \\
  Institute of Automation, Chinese Academy of Sciences \\
  \texttt{bojian.yin@ia.ac.cn}
}
\date{}

\usepackage{mathtools}
\newlength{\factorwd}
\AtBeginDocument{%
  \settowidth{\factorwd}{\scriptsize residual}%
  \addtolength{\factorwd}{3pt}}
\renewcommand{\factor}[2]{%
  \mathmakebox[\factorwd]{\underbrace{#1}_{\smash[b]{\text{\scriptsize #2}}}}}
\def\eqref#1{(\ref{#1})}   
\newcommand{\term}[2]{\underbrace{#1}_{\smash[b]{\text{\scriptsize #2}}}}

\begin{document}

\maketitle

\begin{abstract}

Large language models are trained with backpropagation, whose global gradient coordinates all layers but forces each to hold its activations and wait for the gradient to pass back through every deeper layer. Conventional local learning removes this update locking by training each module to predict the target through its own readout, but has not scaled to billion-parameter pretraining. We identify these private readouts as a key weakness, since they leave each module without information from deeper modules. We propose \textbf{S}hared-\textbf{O}utput \textbf{LO}cal learning (SOLO), which replaces them with a shared, read-only copy of the final module's readout, the only one trained on the output of the whole network. Taken from the previous step, the copy transmits information from the final module without passing gradients between modules or reintroducing update locking. SOLO approaches backpropagation on Transformers of 340M to 2B parameters pretrained on 15B tokens, staying within one point in average zero-shot accuracy with a perplexity gap that narrows with scale. Readout ablations attribute SOLO's improvement over private readouts to sharing. Without update locking, each of $p$ pipeline stages holds activations for $O(1)$ micro-batches instead of $O(p)$. The freed memory permits larger micro-batches, which reach up to 1.44$\times$ the best measured throughput of pipeline on the same backpropagation partition. To our knowledge, SOLO is the first local learning method to show such memory and throughput gains in billion-parameter language-model pretraining. Local learning thus becomes a practical alternative to backpropagation for large-scale pretraining.


\end{abstract}

\section{Introduction}
\label{sec:intro}

End-to-end backpropagation (BP) is both the engine and the bottleneck of
large language model training. Its global gradient, derived from a single
objective at the final output, coordinates the updates of every layer. However,
the same gradient reaches a layer only after propagating back through all
deeper layers, and until then the layer can neither update nor release its
activations. This dependency, known as update
locking~\citep{jaderberg2017dni,belilovsky2020dgl}, creates large memory
overheads and idle time at scale. Activation checkpointing, optimizer sharding, and parallelism strategies reduce training memory overheads, but do not remove the dependency~\citep{chen2016checkpoint,
rajbhandari2020zero,huang2019gpipe,narayanan2019pipedream}. Local learning instead
removes update locking by splitting the network into modules, giving each
module its own learning signal, and passing no gradient between modules
~\citep{belilovsky2019greedy,belilovsky2020dgl,laskin2020parallel,
lowe2019gim}. Each module can then update and release its activations
without waiting for deeper ones. This independence comes at the expense of the
global gradient, so we ask:

\begingroup\setlength{\fboxsep}{3pt}\vspace{1pt}
\noindent\fbox{\parbox{\dimexpr\linewidth-2\fboxsep-2\fboxrule\relax}{\centering\itshape
Can local learning pretrain billion-parameter language models close to BP, and does it improve training efficiency?}}\vspace{1pt}\endgroup

So far, conventional local learning has succeeded mainly in image classification. There,
layer-wise and block-wise methods approach BP accuracy with lower memory
overheads, but only with modest models and few output classes
\citep{belilovsky2019greedy,nokland2019predsim,wang2021infopro,
belilovsky2020dgl,ma2024auglocal,siddiqui2023blockwise}. Language-model pretraining reverses both conditions. It targets much larger models and requires every local readout to map an intermediate state into a large vocabulary. Existing studies of local learning for language models train small models, use limited token budgets, or apply local objectives only after pretraining
\citep{laskin2020parallel,shing2026diffusionblocks,shi2026lopt,
sushma2026hbll}. To our knowledge, no billion-parameter language model has been pretrained with
local learning.

\begin{figure}[t]
\centering
\includegraphics[width=\linewidth]{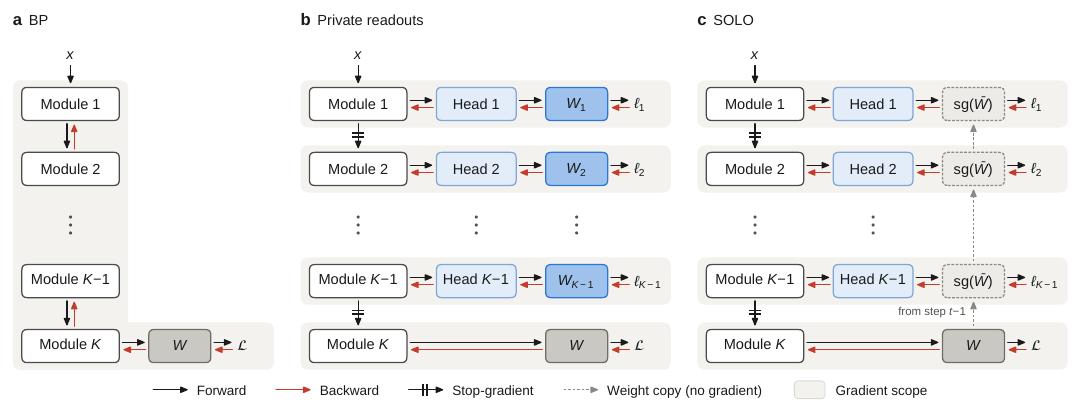}
\caption{Three ways to train a network of $K$ modules. (a) BP trains every
module with the final loss $\mathcal{L}$ through the terminal readout $W$.
(b) Local learning stops the gradient between modules and gives each
auxiliary head a trainable private readout $W_k$ with its own loss $\ell_k$.
(c) SOLO keeps the stop-gradients, but every head predicts through
$\mathrm{sg}(\bar W)$, a read-only copy of $W$ from the previous step. Gradients flow through the copy into its head but do not update it; only $\mathcal{L}$ updates $W$, so no head waits for its current update.}
\label{fig:mechanism}
\vspace{-12pt}
\end{figure}

The main challenge is that removing the global gradient also removes the
coordination it provided. Each module therefore learns from a lightweight
auxiliary head that predicts the network's target and is discarded at
inference \citep{lee2015dsn}. This signal ties the module to the target but carries no
information from deeper modules. Existing methods give each head a private
readout, the final linear map to logits, either learned locally
\citep{belilovsky2019greedy,belilovsky2020dgl,laskin2020parallel,
ma2024auglocal} or fixed at random \citep{mostafa2018deep,yin2025sll}. Local learning, however,
requires isolated gradients, not isolated output readouts. The final module's
\emph{terminal readout} could thus serve every head, giving all modules one
output space of tens of thousands of tokens, far beyond ImageNet's 1,000
classes \citep{deng2009imagenet}. At this size, the placement of tokens in the readout matters, and sharing keeps it consistent across all modules. Language models already share components
across depth, in looped and recurrent Transformers
\citep{dehghani2019universal,geiping2025recurrent,zhu2025ouro} and
early-exit models with a shared readout
\citep{elbayad2020depth,elhoushi2024layerskip}. 
All are trained end to end, so they show the potential of sharing across depth but cannot tell whether it still helps in local learning.

We close this gap by introducing \textbf{SOLO}, \textbf{S}hared-\textbf{O}utput
\textbf{LO}cal learning, which isolates gradients but shares the readout. It
changes a single component of local learning. Every auxiliary head predicts
through a read-only copy of the terminal readout instead of a private one,
which already makes SOLO parameter-efficient (Figure~\ref{fig:mechanism}c).
More importantly, only the final objective updates the terminal readout, so
its previous-step copy transmits information from the final module to every
head. SOLO thus restores a path for information from deeper modules without
reintroducing update locking. We test both halves of this claim, pretraining
Transformers from 340M to 2B parameters on 15B tokens and comparing
pipeline training with BP under the same partition into stages. SOLO
approaches BP at every scale tested. In readout ablations, its intermediate
states become decodable through the terminal readout, and its local
gradients align more closely with BP's than under private readouts. Without
update locking, each of the $p$ pipeline stages holds activations for
$O(1)$ micro-batches instead of $O(p)$. The freed memory permits larger
micro-batches and thus higher pipeline throughput.

We make three contributions. \textbf{Scale.} We demonstrate that SOLO
pretrains language models of 340M to 2B parameters to within one point of
token-matched BP in average zero-shot accuracy. To our knowledge, SOLO is
the first local-learning method to pretrain billion-parameter language
models from scratch. \textbf{Mechanism.} We show that SOLO outperforms
private readouts and identify sharing as the source of the gain. Readout
ablations separate sharing from training, and the gain appears in language
modeling but not in image classification, where even a random readout
costs little. \textbf{Systems.} We show that removing update locking turns
saved activation memory into up to 1.44$\times$ the best measured BP.
Local learning thus pretrains billion-parameter language models close to BP,
and more efficiently under limited memory or bandwidth.
\section{Method}
\label{sec:method}
We factor the gradient that each module receives under BP and under local
learning into a path, a readout, and a residual
(Section~\ref{sec:setup}). Cutting the gradient between modules makes the
path and the residual local but does not require a local readout. SOLO
therefore shares the terminal readout across all modules
(Section~\ref{sec:solo}), which matches the readout factor to BP
(Section~\ref{sec:why}).

\subsection{Local learning}
\label{sec:setup}

Consider a network of $K$ modules, $h_k=f_k(h_{k-1};\theta_k)$ with
$h_0=x$, whose readout $W$ maps the last state to logits $z=Wh_K$. BP trains
all modules on one loss $\mathcal{L}(z,y)$. Local learning instead cuts the
gradient between modules by feeding each module $\mathrm{sg}(h_{k-1})$,
where $\mathrm{sg}$ passes the value but not the gradient. Each module then
learns from its own auxiliary head $\phi_k$ and readout $W_k$, which produce
logits $z_k=W_k\phi_k(h_k)$ and a loss $\ell_k=\mathcal{L}(z_k,y)$. The last
module keeps $W$ and the original loss, and the heads are discarded after
training.

The two rules send different gradients to $h_k$. With the residuals
$\delta=\partial\mathcal{L}/\partial z$ and
$\delta_k=\partial\ell_k/\partial z_k$,
\begin{equation}
\label{eq:fact}
g_k^{\mathrm{BP}}=\nabla_{h_k}\mathcal{L}
=\factor{M_k^{\top}}{path}\factor{W^{\top}}{readout}\factor{\delta}{residual},
\qquad
\hat g_k=\nabla_{h_k}\ell_k
=\factor{J_k^{\top}}{path}\factor{W_k^{\top}}{readout}\factor{\delta_k}{residual},
\end{equation}
where $M_k=\partial h_K/\partial h_k$ and $J_k=\partial\phi_k(h_k)/\partial h_k$.
BP sends every module the same residual through the same readout, and only
the path changes with depth. This gradient is available only after the
forward and backward passes through every deeper module, so module $k$ must
wait before it updates. Existing local learning makes all three factors
local. Avoiding this wait requires it only for the path and the residual,
which BP computes from the deeper modules on the current sample. The
readout is a parameter, not a signal computed on the current sample, so it
need not be local.

\subsection{SOLO: local learning with a shared readout}
\label{sec:solo}

SOLO keeps the local path and residual, but gives every head the terminal
readout $W$. Each head predicts through a read-only copy of it,
\begin{equation}
\label{eq:solo}
z_k=\tau_k\,\mathrm{sg}(\bar W)\,\phi_k(h_k),
\end{equation}
where $\bar W$ is the readout from an earlier step and $\tau_k$ is a learned
scalar temperature that sets the head's logit scale. Only the final loss
updates $W$; each head trains only $\phi_k$ and $\tau_k$. The copy thus
transmits information from the final module to every head without passing a
gradient between modules. The heads also share the map $W$ itself, not only
its range, so every head predicts in one basis (Section~\ref{sec:readout}).
The local gradient of Eq.~\eqref{eq:fact} becomes
\begin{equation}
\label{eq:solo-grad}
\hat g_k=J_k^{\top}\,\tau_k\bar W^{\top}\,\delta_k,
\end{equation}
so each module's residual returns through the same readout as in BP, up to
the scale $\tau_k$ and the delay of $\bar W$. The delay allows modules to run without
waiting, since at step $t$ every head uses $\bar W=W_{t-1}$ while the final
module computes $W_t$. Refreshing the copy only every $S$ steps further reduces
synchronization (Section~\ref{sec:ablations}).

Just as delay avoids waiting, sharing avoids the cost of private readouts.
Each is a $V\times d$ matrix, which is large in language models. With a
32k-token vocabulary and width 2048, it holds 67M parameters, plus their
optimizer states. SOLO replaces all $K-1$ private readouts with read-only
copies that need neither gradients nor optimizer states.
Algorithm~\ref{alg:solo} gives the full training step.

\textbf{Relation to feedback alignment.}
Because each module maps its residual back through $\bar W^{\top}$, SOLO
resembles feedback alignment \citep{lillicrap2016random,nokland2016direct},
sending errors to early layers through an untrained matrix. The
difference lies in the error. Feedback alignment sends the final error of
the current sample, whereas a SOLO module propagates only its own residual
$\delta_k$ through its own head. The copy is still a backward channel, since
$\bar W$ changes with the final module's updates, but it carries parameters
that summarize past batches, not errors on the current sample. Hence no
module receives the derivative of a deeper loss with respect to its own
output, and the modules remain gradient-isolated. For the same reason,
\textsc{rand} (Table~\ref{tab:arms}) controls the prediction map, not a
feedback pathway.

\subsection{How sharing changes the local gradient}
\label{sec:why}

Sharing matches one factor of the local gradient to BP, but not the other
two. For a head that predicts through $\tau_k W_k$, the gap to the BP
gradient splits exactly into a readout, a residual, and a path term,
\begin{equation}
\label{eq:decomp}
\hat g_k-g_k^{\mathrm{BP}}
=\term{J_k^{\top}(\tau_kW_k-W)^{\top}\delta_k}{readout}
+\term{J_k^{\top}W^{\top}(\delta_k-\delta)}{residual}
+\term{(J_k-M_k)^{\top}W^{\top}\delta}{path}.
\end{equation}
SOLO sets $W_k=\bar W$, so the readout term vanishes when $\tau_k=1$ and the
copy is current; the temperature and the delay reintroduce it. The residual
and path terms remain, so sharing does not guarantee a perfect alignment with BP (Appendix~\ref{app:consistency}). The terms are vectors that can partly cancel, so the norm of a term is not
its share of the gap; we use the norms only to compare readout sources. Section~\ref{sec:readout} measures $\cos(\hat g_k,g_k^{\mathrm{BP}})$ on fixed probe batches for each readout
source.

\textbf{Readout sources.}
Table~\ref{tab:arms} changes only the readout of the heads. \textsc{rand}
freezes an independent random readout per head, \textsc{priv} learns one
per head, and SOLO shares the live terminal readout. Two \textsc{pretrained}
variants freeze the terminal readout of a finished SOLO or BP run and
train a fresh model with it. Because these readouts have already seen the
corpus, they are diagnostic probes, not training methods. All variants learn
$\tau_k$, so every head has the same freedom in logit scale.

\begin{table}[t]
\centering
\small
\caption{Readout sources. Each variant changes only the readout of the
auxiliary heads; $\phi_k$ and $\tau_k$ are common to all variants. The last
column counts trainable readout parameters per head.}
\label{tab:arms}
\renewcommand{\arraystretch}{1.2}
\begin{tabular}{@{}llccc@{}}
\toprule
Variant & Readout of head $k$ & Shared & Source & Trainable \\
\midrule
\textsc{rand}
& $W_k^{\mathrm{rand}}$, frozen
& no & random & $0$ \\
\textsc{priv}
& $W_k^{\mathrm{priv}}$, learned
& no & own head & $Vd$ \\
SOLO
& $\bar W$, read-only copy
& yes & live terminal readout & $0$ \\
\textsc{pretrained} (SOLO)
& $W^{\star}_{\mathrm{SOLO}}$, frozen
& yes & fully trained SOLO & $0$ \\
\textsc{pretrained} (BP)
& $W^{\star}_{\mathrm{BP}}$, frozen
& yes & fully trained BP  & $0$ \\
\bottomrule
\end{tabular}
\end{table}

\section{Experiments}
\label{sec:experiments}

We ask whether SOLO pretrains competitive language models at scale, why it
outperforms private readouts, and what removing update locking changes in
pipeline training. Section~\ref{sec:pretraining} compares SOLO with
token-matched BP from 340M to 2B parameters, Section~\ref{sec:readout}
varies only the readout of the auxiliary heads, and
Section~\ref{sec:resources} compares SOLO with pipeline BP under the same
partition. Section~\ref{sec:ablations} varies the refresh period and the
head depth, and Appendix~\ref{app:details} gives the training details.

\subsection{SOLO pretrains models up to 2B near BP quality}
\label{sec:pretraining}

We pretrain Transformers of 340M, 1.3B, and 2B parameters on 15B
SlimPajama tokens, following the training setup of \citet{yang2024gla}, and
evaluate WikiText perplexity and six zero-shot tasks with the
lm-evaluation-harness \citep{gao2023harness}. Both methods train with data
parallelism, and BP is sharded at 2B. Each model has 24 layers, which we
split into two or four modules. Eight modules would leave three layers per
module while adding seven two-block heads, 89\% more compute per token at
340M (Table~\ref{tab:flops}), so we do not use finer splits.


\begin{table}[t]
\caption{Pretraining on 15B SlimPajama tokens. Quality values are
three-seed means; bold marks the best local-learning variant per column.
Throughput and peak memory per GPU are measured on  four A100s at
micro-batch 8 under data parallelism, with BP sharded at 2B. LMB-p and LMB-a
are LAMBADA perplexity and accuracy.}
\label{tab:slim-untie}
\centering
\footnotesize
\setlength{\tabcolsep}{1.6pt}
\begin{tabular}{l l @{\hspace{6pt}} cc @{\hspace{6pt}} ccccccc @{\hspace{6pt}} cc}
\toprule
 & & \multicolumn{2}{c}{\textbf{Perplexity} $\downarrow$}
   & \multicolumn{7}{c}{\textbf{Zero-shot accuracy (\%)} $\uparrow$}
   & \multicolumn{2}{c}{\textbf{Cost}} \\
\cmidrule(lr){3-4} \cmidrule(lr){5-11} \cmidrule(lr){12-13}
\textbf{Scale} & \textbf{Variant} & Wiki. & LMB-p & LMB-a & PIQA & Hella. & Wino. & ARC-e & ARC-c & Avg.
 & tok/s $\uparrow$ & Mem. GB $\downarrow$ \\
\midrule
\rowcolor{refrow}
\multirow{5}{*}{\emph{340M}}
 & BP (DP)                           & 28.04 & 39.5 & 31.8 & 64.2 & 34.6 & 49.9 & 43.4 & 25.0 & 41.5 & 247k & 29.6 \\
 & SOLO, $K{=}2$                     & \textbf{29.05} & \textbf{44.6} & \textbf{29.8} & 64.0 & \textbf{34.0} & 51.9 & 45.6 & 23.6 & \textbf{41.5} & 208k & 15.8 \\
 & \textsc{priv}, $K{=}2$ & 29.47 & 46.3 & 29.2 & \textbf{65.1} & 33.8 & 51.0 & 44.4 & 22.5 & 41.0 & 209k & 15.9 \\
 & SOLO, $K{=}4$                     & 31.19 & 47.9 & 28.6 & 63.3 & 33.1 & 51.3 & \textbf{46.0} & \textbf{23.7} & 41.0 & 156k & 11.2 \\
 & \textsc{priv}, $K{=}4$ & 31.79 & 55.6 & 26.6 & 62.5 & 32.9 & \textbf{53.4} & 44.1 & 23.4 & 40.5 & 157k & 11.5 \\
\midrule
\rowcolor{refrow}
\multirow{4}{*}{\emph{1.3B}}
 & BP (DP)$$              & 21.61 & 20.0 & 38.9 & 68.0 & 41.1 & 53.2 & 48.2 & 25.9 & 45.9 & 76k & 48.9 \\
 & SOLO, $K{=}2$          & \textbf{22.29} & \textbf{22.8} & \textbf{37.8} & \textbf{67.3} & \textbf{39.6} & 52.7 & \textbf{50.1} & \textbf{24.8} & \textbf{45.4} & 77k & 32.7 \\
 & \textsc{priv}, $K{=}2$ & 22.73 & 28.1 & 35.8 & 65.8 & 39.1 & 50.7 & 47.8 & 23.9 & 43.8 & 76k & 33.1 \\
 & SOLO, $K{=}4$          & 23.59 & 23.9 & 37.3 & 66.4 & 38.1 & \textbf{54.1} & 49.5 & 24.4 & 45.0 & 63k & 20.9 \\
\midrule
\rowcolor{refrow}
\multirow{3}{*}{\emph{2B}}
 & BP (FSDP)                         & 20.93 & 22.7 & 39.9 & 67.4 & 41.9 & 51.9 & 50.6 & 25.1 & 46.1 & 62k & 48.0 \\
 & SOLO, $K{=}2$                     & \textbf{21.57} & \textbf{23.0} & \textbf{37.9} & \textbf{67.9} & \textbf{41.6} & \textbf{51.8} & \textbf{49.8} & 25.8 & \textbf{45.8} & 53k & 42.9 \\
 & SOLO, $K{=}4$                     & 22.21 & 23.6 & 37.1 & 66.8 & 40.3 & 51.5 & 49.4 & \textbf{26.5} & 45.3 & 46k & 26.7 \\
\bottomrule
\end{tabular}
\end{table}

\textbf{SOLO stays close to BP.}
With two modules, SOLO approaches BP to within 1.01, 0.68, and 0.64 WikiText perplexity from 340M to 2B, and with four modules to within 3.15, 1.98, and 1.28 (Table~\ref{tab:slim-untie}). Average zero-shot accuracy stays within 0.5 points of BP with two modules and within 0.9 with four. With two
modules, SOLO is within 1.5 points of BP on PIQA, HellaSwag, WinoGrande, and
both ARC sets, and it exceeds BP on ARC-easy at 340M and 1.3B. LAMBADA shows
the largest deficit, 1 to 2 points in accuracy. The auxiliary heads add
11--13\% FLOPs per token with two modules (Table~\ref{tab:flops}).

\textbf{The gap narrows with scale.}
From 340M to 2B, the relative perplexity gap falls from 3.6\% to 3.1\% with
two modules and from 11.2\% to 6.1\% with four. It narrows most on LAMBADA,
where the perplexity gap falls from 5.1 to 2.8 and 0.3 with two modules. The
gap also does not grow with training. SOLO trains as smoothly as BP, without loss spikes, and its gap forms in the first 2B tokens and then stays flat(Figures~\ref{fig:slim-curves} and~\ref{fig:slim-ratio}).

\textbf{Sharing outperforms private readouts at equal cost.}
At the same throughput and memory, sharing lowers WikiText perplexity by
0.42 and 0.44 at 340M and 1.3B with two modules, and by 0.60 at 340M with
four. With two modules, it raises average zero-shot accuracy by 0.5 and 1.6
points. The largest gain is on LAMBADA, where perplexity falls from 28.1 to
22.8 at 1.3B. Section~\ref{sec:readout} traces these gains to sharing.

\textbf{Every module decodes.}
Keeping the heads turns every module into an early exit. In the 1.3B model
with four modules, the exit after the first module already matches the full
340M BP model in WikiText perplexity (28.20 against 28.04;
Table~\ref{tab:exit-vs-lens}). A BP model can be read only at its final
layer. Read through the same readout, its intermediate layers agree with its
final prediction only 16--51\% of the time, against 71--75\% for the SOLO
exits. Each deeper exit refines rather than rewrites the prediction
(Figure~\ref{fig:token-heat}; Appendix~\ref{app:exits}).

\begin{figure}[t]
\centering
\includegraphics[width=0.8\linewidth]{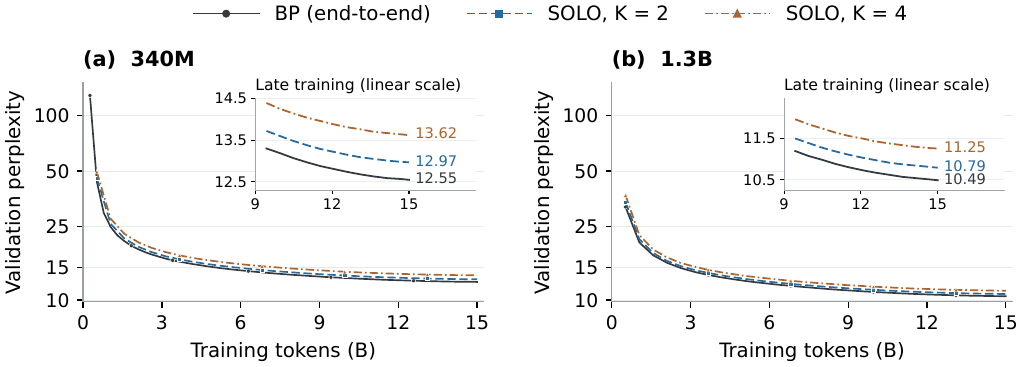}
\caption{Validation perplexity during pretraining at (a) 340M and (b) 1.3B; insets show the
final 9 to 15B tokens on a linear scale.}
\label{fig:slim-curves}
\vspace{-5pt}
\end{figure}

\subsection{SOLO's gain over private readouts comes from sharing}
\label{sec:readout}

\begin{figure}[t]
\centering
\includegraphics[width=0.85\textwidth]{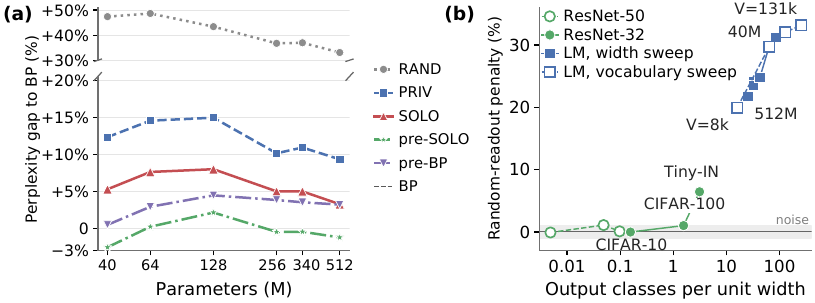}
\caption{Readout source across scale and modality. (a) Perplexity gap to BP
on the one-epoch WikiText-103 sweep with $K=4$, three-seed means.
pre-SOLO and pre-BP are the \textsc{pretrained} variants, which freeze the
terminal readout of a fully trained SOLO or BP model. (b) Cost of a random
readout relative to a private one, as the relative increase in validation
perplexity (language) or test error (vision), against the number of output
classes per unit width. Numbers
are in Appendix~\ref{app:readout-extra}.}
\label{fig:readout-ladder}
\end{figure}

We next identify the source of SOLO's gain over private readouts. The gain
is robust. SOLO leads \textsc{priv} by 2.1-4.7 perplexity from 40M to 512M
parameters on the one-epoch WikiText-103 setting
(Figure~\ref{fig:readout-ladder}a), and by 0.15 and 0.42 with two and four
modules in the Transformer-XL 30-epoch training
\citep{dai2019transformerxl}, with 4.8M rather than 16.7M added parameters
at two modules (Table~\ref{tab:wt103-pubrecipe-sp}). The two readouts
differ along two axes. A private readout is trained by its own head and
used by that head alone, whereas the SOLO readout is trained by the final
module and shared by every head. We vary the two axes independently.

\textbf{The best readout is learned together with local modules.}
In \textsc{pretrained} variants, every head predicts from the first
step through the frozen terminal readout of a fully trained model
(Table~\ref{tab:arms}). First, the trained SOLO readout outperforms
the live copy on the one-epoch sweep with four modules
(Figure~\ref{fig:readout-ladder}a). After 30 epochs with two modules, the
live copy has matured and is best (Table~\ref{tab:wt103-pubrecipe-sp}). In
the same one-epoch runs, the live copy reaches the same gradient alignment
after roughly 2,000 steps (Figure~\ref{fig:align}a). A readout therefore
helps more once the final module has learned it. Second, the fully trained
BP readout, despite coming from the stronger model, trails the fully
trained SOLO readout by 1.0-2.0 perplexity, and only the latter makes the
module outputs directly decodable (Figure~\ref{fig:lens}). A useful readout
must therefore also be learned together with local modules, a condition that the live copy of SOLO meets.

\textbf{Sharing helps through a common basis.}
Table~\ref{tab:readout-arms} compares shared and unshared versions of the trained
terminal readout and of a random matrix on the 40M model. Rotating the
trained readout by a different orthogonal matrix for each module
(\textsc{rot}) keeps its content and the logits each head can reach, but
raises perplexity by up to 13.3 over SOLO. Sharing one random matrix across
all heads (\textsc{rshare}), instead of drawing one per head, lowers
perplexity by up to 12.0. A common basis thus helps even without trained
content. Through the skip connections each module inherits the basis
of its input. A shared readout decodes this basis directly, whereas a
rotated one forces each later module to translate it, and their local
gradients lose alignment with BP (Appendix~\ref{app:readout-extra}). Deeper heads
do not close the gap between SOLO and PRIV, so head capacity cannot
substitute for a shared readout.

\begin{table}[t]
\caption{Readout sources under a published WikiText-103 recipe, mean over three
seeds. Perplexity per subword token, word-level in parentheses; $\Delta$par
denotes trainable parameters beyond BP's 62.3M, and $^{\S}$ an additional
11.9M-parameter frozen pretrained readout.}
\label{tab:wt103-pubrecipe-sp}
\centering
\footnotesize
\setlength{\tabcolsep}{3pt}
\begin{tabular}{c l @{\hspace{6pt}} cc @{\hspace{6pt}} ccc}
\toprule
\multirow{2}{*}{$K$} & \multirow{2}{*}{\textbf{Variant}}
 & \multicolumn{2}{c}{\textbf{Perplexity}}
 & \multicolumn{3}{c}{\textbf{Cost}} \\
\cmidrule(lr){3-4}
\cmidrule(lr){5-7}
 & & Val $\downarrow$ & Test $\downarrow$
 & tok/s $\uparrow$ & Peak mem (GB) $\downarrow$ & $\Delta$par (M) \\
\midrule
\rowcolor{refrow}
1 & BP (end-to-end)
  & 23.15 (38.30) & 23.59 (41.59) & 88k & 16.6 & --- \\
\midrule
\multirow{4}{*}{2}
 & SOLO
  & \textbf{23.49} (38.95) & \textbf{24.02} (42.48)
  & 53k & 12.9 & $+4.8$ \\
 & \textsc{priv}
  & 23.64 (39.25) & 24.20 (42.87)
  & 52k & 13.0 & $+16.7$ \\
 & \textsc{pretrained} (SOLO)
  & 23.59 (39.10) & 24.06 (42.56)
  & 52k & 13.0 & $+4.8^{\S}$ \\
 & \textsc{pretrained} (BP)
  & 23.82 (39.59) & 24.30 (43.07)
  & 52k & 12.9 & $+4.8^{\S}$ \\
\midrule
\multirow{2}{*}{4}
 & SOLO
  & \textbf{24.35} (40.62) & \textbf{24.82} (44.17)
  & 29k & 10.6 & $+14.4$ \\
 & \textsc{priv}
  & 24.77 (41.42) & 25.24 (45.05)
  & 29k & 11.1 & $+50.2$ \\
\bottomrule
\end{tabular}
\end{table}

\textbf{Sharing brings local gradients closer to BP.}
As predicted in Section~\ref{sec:why}, sharing eliminates the readout term
of the gap to the BP gradient while leaving the residual and path terms
(Figure~\ref{fig:align}c). Local gradients also align better with BP. Their
cosine with the BP gradient rises from 0.32 under \textsc{rand} to 0.52
under \textsc{priv} and 0.70 under SOLO, tracking perplexity, and the order
holds at every module (Figure~\ref{fig:align}a,b). Alignment alone does not
explain perplexity, since SOLO and \textsc{pretrained} (SOLO) reach similar
alignment yet differ in perplexity (Figure~\ref{fig:align}d).

\begin{figure}[t]
\centering
\includegraphics[width=\textwidth]{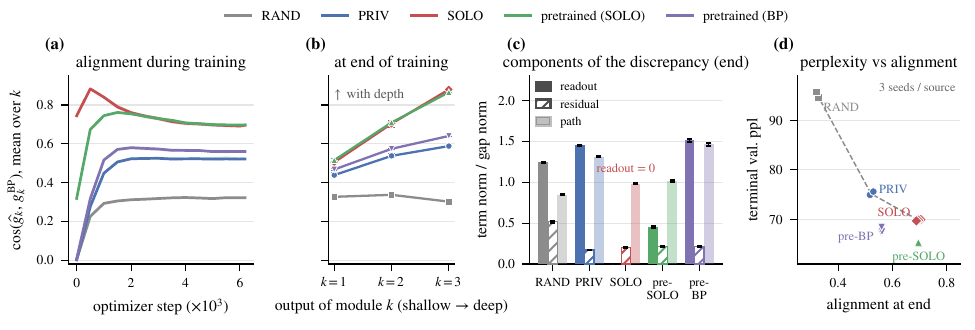}
\caption{Gradient alignment on the 40M WikiText-103 sweep, three seeds per
variant. (a) Alignment during training; (b) alignment at the end, by
module; (c) norms of the three terms of Eq.~\eqref{eq:decomp} at the end,
divided by the norm of the gap; (d) terminal perplexity against alignment.
The terms can partly cancel, so their norms are not shares of the gap.
Setting aside $\tau_k$, the readout term is zero under SOLO by
construction.}
\label{fig:align}
\end{figure}

\textbf{Sharing helps in language, where the vocabulary is large, but not in vision.}
From CIFAR to ImageNet-1k, sharing yields no consistent gain over
\textsc{priv} (Appendix~\ref{app:readout-extra}). The readout itself
matters far less there. Relative to a private readout, a random one costs
almost nothing when a classifier has fewer classes than feature dimensions,
but up to 33\% in language, where the vocabulary is 26--85 times the model width
(Figure~\ref{fig:readout-ladder}b). Sharing pays off only where the readout
matters.

\subsection{SOLO turns saved activation memory into pipeline throughput}
\label{sec:resources}

Update locking matters most in pipeline parallelism, where each GPU holds
one stage, a group of consecutive layers. We therefore compare SOLO with
pipeline BP under the same partition into $p{=}K$ stages, one module each,
which isolates the effect of removing update locking. We train a 96-layer,
1.2B-parameter model on eight A100 GPUs with $M{=}24$ or $72$
micro-batches per step. The BP baselines are the one-forward-one-backward
schedule (1F1B), which alternates forward and backward passes, and its
interleaved variant with $v$ virtual stages per GPU
(VPP; \citealp{narayanan2021efficient}) (Figure~\ref{fig:same-split};
Appendix~\ref{app:pipeline-bp}).

\textbf{Each stage holds activations for $O(1)$ micro-batches instead of $O(p)$.}
Under 1F1B, each stage keeps the activations of a micro-batch until its
gradient returns from the later stages. The first stage therefore holds $p$
micro-batches at once (Figure~\ref{fig:same-split}a), and its activation
memory equals that of the unsplit model, however many stages are added. A
SOLO stage runs its local backward pass right after its forward pass and
releases the activations. It holds a single micro-batch
(Figure~\ref{fig:same-split}b), so its activation memory falls as $1/p$. At
$p{=}8$, peak activation memory drops from 17.4 to 2.5\,GB, and for
$p = 2, 4, 6, 8$ the reduction is 1.9, 3.7, 5.4, and 7.0-fold, close to $p$
(Figure~\ref{fig:same-split}d; Table~\ref{tab:pipe-mem}). Pipeline BP
reaches this level only by recomputing every block (2.1\,GB), at three quarters of its throughput,
whereas VPP raises activation memory to 24.8\,GB to
shrink the bubble (Appendix~\ref{app:pipeline-bp}).

\textbf{Without update locking, the pipeline has no backward bubble.}
In 1F1B, a stage cannot start a backward pass until the next stage returns
a gradient, so the pipeline sits idle for a fraction $(p-1)/(M+p-1)$ of
each step, known as the bubble. SOLO stages never wait for a gradient, so
this bubble disappears. In exchange, each stage runs its auxiliary head,
whose share of the work shrinks as the stage holds more layers. With 12
layers per stage at $p{=}8$, the heads add about 1.4\% of a step per
additional module, and SOLO is faster whenever $M$ is below about 70
(Appendix~\ref{app:pipeline-bp}). With a 50-step refresh, it reaches
$1.18\times$ the throughput of 1F1B at $M{=}24$, matches it at $M{=}72$, and
beats the faster VPP schedule by 10\% at $M{=}24$
(Figure~\ref{fig:same-split}c). Per-step refresh synchronizes the stages
and costs up to 7\%, which a 10-step refresh recovers. On the 24-layer
pretraining models, with 3 to 12 layers per stage, SOLO reaches
0.72--0.91$\times$ 1F1B at a fixed micro-batch
(Figure~\ref{fig:pipe-three}b).

\begin{figure}[t]
\centering
\includegraphics[width=\textwidth]{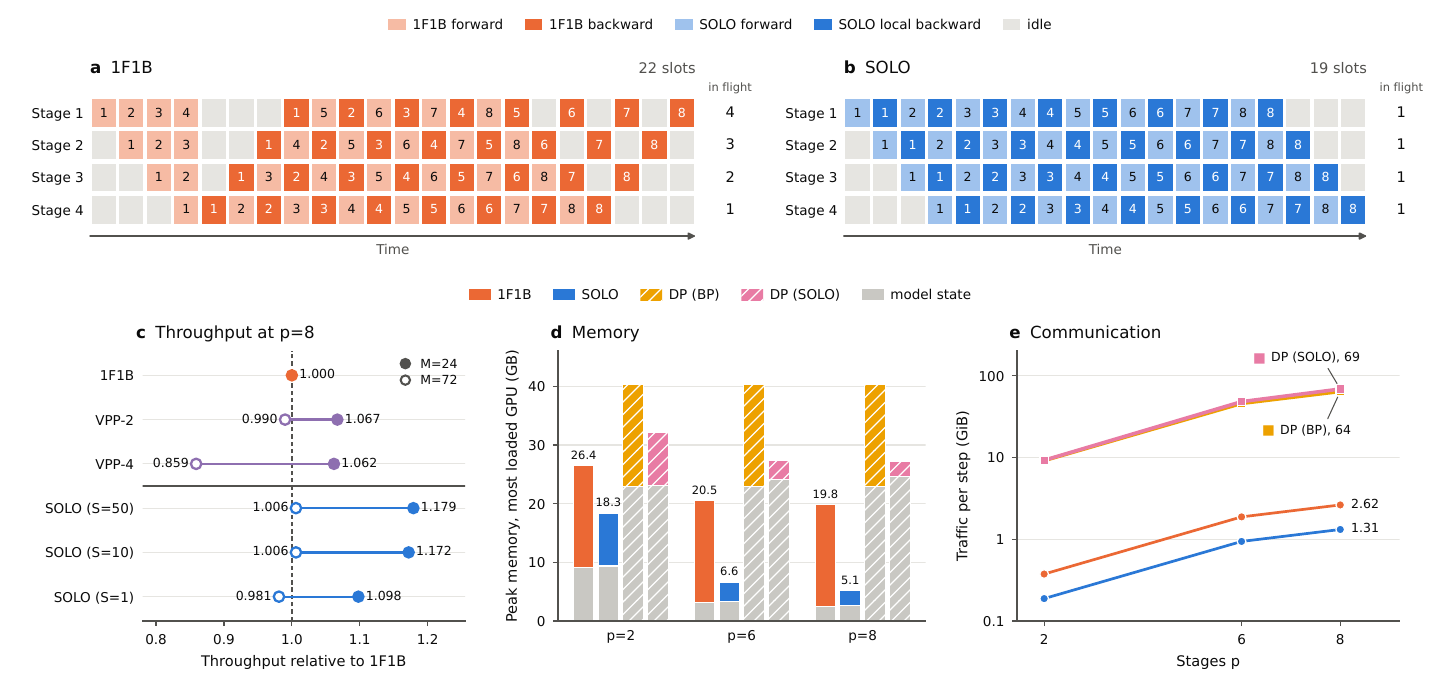}
\caption{SOLO against pipeline BP under the same partition. (a, b)
Schedules of 1F1B and SOLO with $p{=}4$ stages and eight micro-batches,
ignoring the auxiliary heads. Each cell is one forward or backward pass of
one micro-batch; 1F1B takes 22 slots and SOLO 19, and the right column
gives the peak number of micro-batches in flight per stage. Only one step
is drawn; with $S{>}1$, SOLO's idle slots at the start and end overlap with
the neighboring steps. (c) Throughput relative to 1F1B at $p{=}8$; VPP-$v$
is interleaved 1F1B with $v$ virtual stages per GPU, and $S$ is the refresh
period of the copy in steps. (d) Peak memory on the most loaded GPU, with
model state in gray and activations in color. (e) Communication volume per
step. Panels (d) and (e) use $M{=}24$; data-parallel runs are shown for
reference.}
\label{fig:same-split}
\end{figure}

\textbf{The freed memory becomes throughput.}
At a fixed global batch, a larger micro-batch uses the GPU more efficiently
but leaves fewer micro-batches per step. 1F1B cannot exploit this trade,
because a larger micro-batch both raises its activation memory and widens
its bubble, so its throughput peaks at micro-batch 4. SOLO has no bubble
and a small activation footprint, so its throughput rises until the GPU
kernels saturate, reaching $1.44\times$ the best 1F1B throughput at
micro-batch 16 and $1.43\times$ at equal memory (Table~\ref{tab:mbsweep};
Figure~\ref{fig:mbsweep}).

\textbf{SOLO tolerates slow links between stages.}
SOLO propagates activations forward and no gradients backward, which halves the
traffic between stages (Figure~\ref{fig:same-split}e). Since no stage waits
for a gradient, a slow link also cannot stall the backward pass. On a
24-layer model split into two stages, throttling the link to 1\,Gb/s
reduces the throughput of SOLO by 1.2\% and that of 1F1B by 51\%
(Table~\ref{tab:link}).

Local learning thus relaxes the memory, micro-batch, and bandwidth limits
of pipelines, and schedules designed for it may widen these gains across
nodes.

\subsection{Ablations}
\label{sec:ablations}

\textbf{Refresh period.}
The pretraining runs of Table~\ref{tab:slim-untie} read the copy at every
step ($S{=}1$), but the copy can also be refreshed less often. On WikiText-103
with four modules, periods up to $S{=}50$ change validation perplexity by
at most 0.4\%, while $S{=}100$ and $S{=}200$ raise it by 1.9\% and 3.0\%
(Appendix~\ref{app:sync}). A period of 10 to 50 steps therefore keeps
quality and recovers the 7\% throughput that per-step refresh costs in
pipelines (Section~\ref{sec:resources}).

\textbf{Auxiliary-head depth.}
Deeper heads and sharing help independently. On WikiText-103 with four
modules, from 40M to 256M parameters, deepening SOLO's heads from a linear
readout ($H{=}0$) to three blocks shrinks its perplexity gap to BP from
8.6--13.5\% to 2.7-5.2\% (Table~\ref{tab:hdgrid}). \textsc{priv} improves
along the same curve but stays 5.4-7.8\% behind SOLO in all sixteen cells,
so even at its best depth it trails SOLO with one-block heads by 2.3--2.4\%
(Figure~\ref{fig:hdgrid}). Beyond $H{=}2$, each extra head block adds the
computation of $K{-}1$ layers and narrows the gap by only about one point
(Figure~\ref{fig:pipe-three}c), so we use $H{=}2$.

\section{Related work}
\label{sec:related}

\textbf{Local learning.}
Local learning approaches BP in image classification
\citep{belilovsky2019greedy,nokland2019predsim,belilovsky2020dgl,
siddiqui2023blockwise,ma2024auglocal,yin2025sll}, but language-model
studies stop at a 6M-parameter Transformer on LM1B \citep{laskin2020parallel}, 12-layer Llama-2-style models \citep{shing2026diffusionblocks}, or 774M parameters over repeated
passes of WikiText-103 \citep{sushma2026hbll}, or apply local objectives
only after pretraining \citep{shi2026lopt}. These methods refine the local
objectives and heads \citep{wang2021infopro,pyeon2021sedona} but maintain a
private readout per module, learned or fixed at random
\citep{mostafa2018deep}. SOLO instead shares the terminal readout across
modules.

\textbf{Other ways to avoid update locking.}
Delayed and synthetic gradients, forward gradients, and LocoProp pass a
downstream derivative, an estimate of it, or a target derived from it
\citep{huo2018ddg,jaderberg2017dni,baydin2022gradients,qin2026splitfg,
amid2022locoprop}, and feedback alignment sends the final error, or the
label, through a fixed or learned matrix \citep{lillicrap2016random,
akrout2019wm,launay2020dfa}. Forward-only rules avoid derivatives but have
not reached language-model pretraining \citep{hinton2022forwardforward,
dellaferrera2022error}. SOLO passes no error or derivative between modules,
only a copy of a parameter (Section~\ref{sec:solo}).

\textbf{Sharing a readout across depth.}
Early-exit language models reuse or align one readout across depth
\citep{elbayad2020depth,schuster2022calm,elhoushi2024layerskip}, but train
end to end, so every exit loss updates the readout and the layers below.
Lenses decode intermediate layers after training
\citep{nostalgebraist2020logitlens,belrose2023tuned}, and aligned training
makes one classifier serve several layers \citep{jiang2024aligned}. SOLO
shares the readout during training and updates it only with the final
objective. Fixed random classifiers suffice for image classification
\citep{hoffer2018fix}, but not for language in our experiments.

\textbf{Pipeline parallelism.}
Schedules such as GPipe, 1F1B, and its interleaved variant reduce idle time
between stages \citep{huang2019gpipe,narayanan2019pipedream,
narayanan2021efficient}, and zero-bubble schedules fill the rest by
splitting the backward pass \citep{qi2024zerobubble}. All keep update
locking, so each stage must hold its activations, or recompute them
\citep{chen2016checkpoint}, until its gradient returns. Local-learning and
interlocking pipelines have mainly targeted image classifiers
\citep{gomez2022interlocking,guo2024ppll}. DiLoCo reduces the synchronization
across replicas \citep{douillard2023diloco}, while SOLO eliminates the wait
for gradients across depth.

\section{Discussion and Conclusion}
\label{sec:discussion}

Local learning usually isolates both the gradients and the readouts of its
modules. Our results show that only the gradients need isolation. A shared
readout gives all modules one basis, restores information flow from the
final module, and brings local gradients closer to BP. It helps most when
learned with local modules and when the vocabulary is large, as in
language. Sharing across depth, used by looped and early-exit models under
end-to-end training, thus also works without gradients between modules.

\textbf{Limitations.}
We test two and four modules, since eight would leave only three layers per
module in our 24-layer models. Finer partitions add more heads, and their
trade-off between quality and efficiency remains open. Our comparisons
match training tokens, while the heads add 11--13\% FLOPs per token with
two modules. The systems study runs on one node, and on the 24-layer
pretraining models SOLO trails 1F1B at a fixed micro-batch. Finally, we
study plain Transformers without mixture-of-experts layers, grouped-query
attention, long contexts, or post-training, and we give no convergence
guarantee.

\textbf{Future work.}
Cheaper heads would make finer partitions affordable, and refreshing the
copy asynchronously would remove the remaining synchronization between
stages and let pipelines span nodes, where SOLO's tolerance of slow links
matters most. The pretrained readouts show that a mature copy helps, and
schedules such as an exponential moving average of the readout may bring
this benefit into a single run. Extending SOLO to modern architectures and
post-training would test its scope. Isolation across depth may also
complement low-communication training across replicas
\citep{douillard2023diloco}, which reduces synchronization along the other
axis.

\textbf{Conclusion.}
To our knowledge, SOLO is the first local-learning method to pretrain
billion-parameter language models from scratch. Sharing a read-only copy of
the terminal readout gives every module shared basis while keeping
gradients isolated, and brings local learning close to BP
from 340M to 2B parameters. Without update locking, each pipeline stage
holds activations for $O(1)$ micro-batches, and the freed memory becomes
throughput. Local learning thus becomes a practical option for pretraining when memory
or interconnect bandwidth is limited.

\section*{Acknowledgments}
We thank Sander Bohte, Lorenzo Pes, and Chenxi Dou for helpful discussions.

\bibliography{referencesv1}
\bibliographystyle{plainnat}

\appendix

\section{The SOLO step}
\label{app:algorithm}

Algorithm~\ref{alg:solo} states one optimizer step, where Step applies one
optimizer update to the parameters it lists. A module updates as soon as
its own forward and local backward passes have run, before the next module
starts, so no module waits for a deeper one. In one process, the modules
run in sequence and $\bar W$ is a view of $W$ with the gradient stopped,
which gives the previous-step copy of Section~\ref{sec:solo} at no cost.
In the pipelines of Section~\ref{sec:resources}, each module runs on its
own GPU, activations stream forward, and line~1 becomes a one-way broadcast
of $W$ from the last GPU every $S$ steps. No gradient is passed between
modules.

\begingroup
\makeatletter
\setlength{\columnwidth}{0.72\textwidth}%
\renewcommand\float@makebox[1]{%
  \vbox{\hsize\textwidth
    \moveright\dimexpr(\textwidth-#1)/2\relax
    \vbox{\hsize=#1 \@parboxrestore
      \@fs@pre\@fs@iftopcapt
        \ifvoid\@floatcapt\else\unvbox\@floatcapt\par\@fs@mid\fi
        \unvbox\@currbox
      \else\unvbox\@currbox
        \ifvoid\@floatcapt\else\par\@fs@mid\unvbox\@floatcapt\fi
      \fi\par\@fs@post\vskip\z@}}}%
\makeatother
\begin{algorithm}[!htb]
\caption{SOLO, one optimizer step with $K$ gradient-isolated modules.
\trainpar{Blue} marks a parameter this step updates, \freezepar{orange} the read-only copy
$\freezepar{\bar W}$ that no local loss updates.}
\label{alg:solo}
\renewcommand{\algorithmicrequire}{\textbf{Input:}}
\begin{algorithmic}[1]

\Require modules $f_1,\dots,f_K$ with parameters $\trainpar{\theta_k}$; auxiliary heads $\trainpar{(\phi_k,\tau_k,b_k)}$ for $k<K$; terminal readout $\trainpar{(W,b)}$; minibatch $(x,y)$
\State $\freezepar{\bar W} \gets \mathrm{sg}(\trainpar{W})$ for all $k<K$ \Comment{read-only copy}
\State $h_0 \gets x$
\For{$k = 1,\dots,K-1$} \Comment{module $k$ starts}
  \State $h_k \gets f_k(\mathrm{sg}(h_{k-1});\,\trainpar{\theta_k})$
  \State $\ell_k \gets \mathrm{CE}\big(\mathrm{softmax}(\trainpar{\tau_k} \freezepar{W_k}\,\trainpar{\phi_k}(h_k)+\trainpar{b_k}),\,y\big)$
  \State $\trainpar{(\theta_k,\phi_k,\tau_k,b_k)} \gets \mathrm{Step}(\nabla\ell_k)$ \Comment{module $k$ done}
\EndFor
\State $h_K \gets f_K(\mathrm{sg}(h_{K-1});\,\trainpar{\theta_K})$
\State $\mathcal{L} \gets \mathrm{CE}\big(\mathrm{softmax}(\trainpar{W} h_K + \trainpar{b}),\,y\big)$
\State $\trainpar{(\theta_K, W, b)} \gets \mathrm{Step}(\nabla\mathcal{L})$ \Comment{update of $W$}


\end{algorithmic}
\end{algorithm}
\endgroup

\section{The local gradient when readout and path match}
\label{app:consistency}

Equation~\eqref{eq:decomp} splits the gap between the local and the BP
gradient into a readout, a residual, and a path term. This appendix shows
what remains when the readout and path terms vanish. For a head that
predicts through $\tau_kW_k$, Eq.~\eqref{eq:fact} gives
$\hat g_k=J_k^{\top}(\tau_kW_k)^{\top}\delta_k$.

\begin{proposition}
\label{prop:descent}
Let $\tau_kW_k=W$ and $J_k=M_k$. Then
$\langle\hat g_k,g_k^{\mathrm{BP}}\rangle=\delta_k^{\top}G_k\,\delta$ with
$G_k=WM_kM_k^{\top}W^{\top}\succeq0$. If moreover $\delta_k=\delta$, then
$\hat g_k=g_k^{\mathrm{BP}}$.
\end{proposition}
\begin{proof}
By Eq.~\eqref{eq:fact},
$\langle\hat g_k,g_k^{\mathrm{BP}}\rangle
=\delta_k^{\top}(\tau_kW_k)J_kM_k^{\top}W^{\top}\delta$.
Substituting $\tau_kW_k=W$ and $J_k=M_k$ gives the first claim. With
$\delta_k=\delta$ as well, all three terms of Eq.~\eqref{eq:decomp} vanish,
which gives the second.
\end{proof}

The inner product can still be negative, because a positive semidefinite
form can be negative off its diagonal. Writing
$\|v\|_{G_k}=\|M_k^{\top}W^{\top}v\|$, so that
$\|\delta\|_{G_k}=\|g_k^{\mathrm{BP}}\|$, and expanding
$\delta_k=\delta+(\delta_k-\delta)$ gives, by the Cauchy--Schwarz inequality
for $G_k$,
\begin{equation}
\label{eq:descent-bound}
\langle\hat g_k,g_k^{\mathrm{BP}}\rangle
=\|g_k^{\mathrm{BP}}\|^{2}+(\delta_k-\delta)^{\top}G_k\,\delta
\;\ge\;\|\delta\|_{G_k}\bigl(\|\delta\|_{G_k}-\|\delta_k-\delta\|_{G_k}\bigr).
\end{equation}
The local gradient thus has a positive inner product with the BP gradient
whenever $\|\delta_k-\delta\|_{G_k}<\|\delta\|_{G_k}$. Under cross-entropy,
$\delta_k-\delta=p_k-p$, the difference between the distributions predicted
by the head and by the final module, so the condition asks the head to
predict close to the final module.

SOLO meets the readout condition up to $\tau_k$ and the lag of $\bar W$
(Section~\ref{sec:why}). The path condition does not hold in general, since
the head is much shallower than the modules after module $k$. Alignment at
$h_k$ also does not imply alignment in the parameters $\theta_k$ unless the
two gradients are equal. Section~\ref{sec:readout} therefore measures the
alignment directly (Figure~\ref{fig:align}).  

\section{Experimental details}
\label{app:details}

\paragraph{Language.}
The pretraining runs follow the Transformer++ recipe of \citet{yang2024gla},
a single pass over 15B SlimPajama tokens \citep{soboleva2023slimpajama} at
340M, 1.3B, and 2B parameters. Neither method ties the input embedding to
the readout \citep{press2017tying}. Each auxiliary head stacks $H$
Transformer++ blocks and a final RMSNorm, with $H{=}2$ unless stated
otherwise. Both methods train with data parallelism on four A100s. BP is
replicated (DDP) at 340M and 1.3B and sharded (FSDP) at 2B, where
replicated training exceeds an 80\,GB device. SOLO is replicated at every
scale. Each GPU holds all modules and heads, runs the modules in sequence,
and releases the activations of each module after its local backward pass.
The heads read $\bar W$ from the previous step ($S{=}1$).
Table~\ref{tab:slim-untie} thus compares the two methods under data
parallelism, and Section~\ref{sec:resources} compares them under pipeline
parallelism with the same partition. We evaluate with the
lm-evaluation-harness \citep{gao2023harness} on WikiText perplexity, the
word-level perplexity on the WikiText-2 test set \citep{merity2017pointer},
and six zero-shot tasks. Accuracy is length-normalized for HellaSwag and
ARC-challenge and raw for the other tasks. The readout sweep trains for one
epoch on WikiText-103 at six sizes from 40M to 512M parameters, each split
into $K{=}4$ modules with two-block heads, and refreshes the shared copy
every 100 steps. The head-depth grid of Table~\ref{tab:hdgrid} refreshes
it every step.

\paragraph{Gap along training.}
Figure~\ref{fig:slim-ratio} divides SOLO's validation perplexity by BP's at the same token count
along the runs of Figure~\ref{fig:slim-curves}. The gap is largest in the first billion tokens
and settles by 2B; from there to 15B it stays within half a point at $K{=}2$ (2.9 to 3.3\% at 340M, 2.4 to 2.9\% at 1.3B)
and within 1.2 points at $K{=}4$ (7.9 to 9.1\% and 6.5 to 7.2\%).

\begin{figure}[!htb]
\centering
\includegraphics[width=0.9\textwidth]{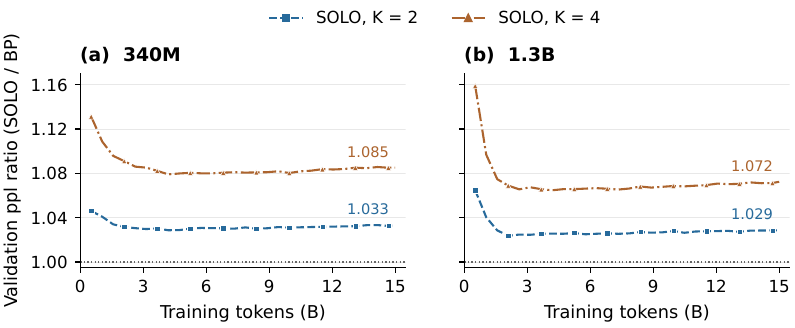}
\caption{Relative perplexity gap during pretraining, SOLO validation perplexity divided by BP's at
the same token count, for (a) 340M and (b) 1.3B, from the single-seed training-log curves of
Figure~\ref{fig:slim-curves}; ratios at matching evaluation steps, no interpolation or smoothing;
the dotted line is parity.}
\label{fig:slim-ratio}
\end{figure}

\paragraph{Vision.}
The vision sweep trains ResNet-32 on CIFAR-10, CIFAR-100, and Tiny-ImageNet, split into
fifteen modules, each non-terminal module with a four-layer
convolutional auxiliary head with a 64-dimensional readout; single-seed ResNet-50 runs repeat
the three datasets at a readout width of 2048. Per-dataset schedules differ, so comparisons
across datasets are qualitative.

\section{Readout sweeps}
\label{app:readout-extra}


Tables~\ref{tab:lm-ladder-scale} and~\ref{tab:vision-ladder} give the numbers behind
Figure~\ref{fig:readout-ladder}; Table~\ref{tab:imagenet} extends the vision comparison to
ImageNet-1k.

\begin{table}[!htb]

  \caption{Language readout sweep across scale, WikiText-103 validation
  perplexity (word-level, $V{=}32{,}768$, one epoch, $K{=}4$), three-seed
  means. The sub-row gives $V/d$, which falls from 85 to 26 from left to
  right. The shaded BP row is the end-to-end reference, and bold marks the
  best local-learning variant per column. Rows run from worst to best, an
  order that holds at every size. The two \textsc{pretrained} variants
  freeze the terminal readout of a fully trained SOLO or BP model (pre-SOLO
  and pre-BP in Figure~\ref{fig:readout-ladder}); they are diagnostic
  probes, not training methods.}
  
\label{tab:lm-ladder-scale}
\centering
\footnotesize
\setlength{\tabcolsep}{8pt}
\begin{tabular}{l cccccc}
\toprule
 & \multicolumn{6}{c}{\textbf{Validation perplexity} $\downarrow$} \\
\cmidrule(lr){2-7}
\textbf{Variant} & \textbf{40M} & \textbf{64M} & \textbf{128M} & \textbf{256M} & \textbf{340M} & \textbf{512M} \\
 & $\scriptstyle V/d{=}85$ & $\scriptstyle 64$ & $\scriptstyle 43$ & $\scriptstyle 32$ & $\scriptstyle 32$ & $\scriptstyle 26$ \\
\midrule
\rowcolor{refrow}
BP (end-to-end)  & 67.08 & 54.91 & 44.46 & 40.47 & 39.51 & 36.81 \\
\textsc{rand}    & 98.88 & 81.63 & 63.76 & 55.34 & 54.12 & 48.99 \\
\textsc{priv}    & 75.34 & 62.89 & 51.10 & 44.57 & 43.82 & 40.24 \\
SOLO             & 70.62 & 59.09 & 48.01 & 42.49 & 41.48 & 38.01 \\
\textsc{Pretrained}, BP  & 67.40 & 56.51 & 46.44 & 42.02 & 40.90 & 37.99 \\
\textsc{Pretrained}, SOLO & \textbf{65.37} & \textbf{55.03} & \textbf{45.41} & \textbf{40.28} & \textbf{39.33} & \textbf{36.36} \\
\bottomrule
\end{tabular}
\end{table}

\begin{table}[!htb]
\caption{Vision readout sweep, terminal test accuracy (\%), backbones split into 15
gradient-isolated modules. ResNet-32 columns are three-seed means; ResNet-50 columns are
single-seed scouting runs ($^{\dagger}$). Datasets are ordered by classes per unit width, C/d within each
backbone. Bold: best per column.
$^{\ddagger}$Trainable parameters over BP on ResNet-50 Tiny-ImageNet, millions.}
\label{tab:vision-ladder}
\centering
\footnotesize
\setlength{\tabcolsep}{4pt}
\begin{tabular}{l ccc ccc c}
\toprule
 & \multicolumn{3}{c}{\textbf{ResNet-32} ($d{=}64$)} & \multicolumn{3}{c}{\textbf{ResNet-50} ($d{=}2048$)$^{\dagger}$} & \\
\cmidrule(lr){2-4}\cmidrule(lr){5-7}
\textbf{Variant} & \textbf{C-10} & \textbf{C-100} & \textbf{Tiny-IN} & \textbf{C-10} & \textbf{C-100} & \textbf{Tiny-IN} & \textbf{$\Delta$par}$^{\ddagger}$ \\
 & $\scriptstyle C/d{=}0.16$ & $\scriptstyle 1.56$ & $\scriptstyle 3.1$ & $\scriptstyle 0.005$ & $\scriptstyle 0.049$ & $\scriptstyle 0.098$ & $\scriptstyle \mathrm{R50,\,M}$ \\
\midrule
\textsc{rand}    & 92.53 & 67.28 & 44.30 & \textbf{95.80} & 78.56 & 63.54 & $+41.1$ \\
\textsc{priv}    & 92.51 & \textbf{67.97} & \textbf{47.35} & 95.71 & 79.43 & 63.60 & $+47.2$ \\
SOLO             & \textbf{92.57} & 67.22 & 46.14 & 95.34 & 79.13 & 63.31 & $+41.1$ \\
\textsc{Pretrained} SOLO & 92.27 & 67.63 & 46.48 & 95.43 & \textbf{79.55} & \textbf{64.59} & $+41.1$ \\
\bottomrule
\end{tabular}
\end{table}

\begin{table}[!htb]
\caption{ImageNet-1k ($224^2$, batch 256, 90 epochs; mean over three seeds). Memory change
is relative to the end-to-end baseline of the same backbone. Bold: best per column within
each backbone.}
\label{tab:imagenet}
\centering
\footnotesize
\setlength{\tabcolsep}{5pt}
\begin{tabular}{llccc}
\toprule
\textbf{Backbone} & \textbf{Variant} & \textbf{Split} & \textbf{Top-1} & \textbf{Peak mem} \\
\midrule
\multirow{5}{*}{ResNet-50}
 & BP                     & ---            & \textbf{76.49} & 25.3\,GB \\
\cmidrule(l){2-5}
 & \textsc{priv}, $K{=}2$ & $[5\,|\,11]$   & 76.27          & 18.3\,GB ($-28\%$) \\
 & SOLO, $K{=}2$          & $[5\,|\,11]$   & 76.22          & 18.3\,GB ($-28\%$) \\
 & \textsc{priv}, $K{=}4$ & $[2,3,6,5]$    & 75.15          & 14.0\,GB ($-45\%$) \\
 & SOLO, $K{=}4$          & $[2,3,6,5]$    & 75.03          & \textbf{13.7}\,GB ($-45\%$) \\
\midrule
\multirow{3}{*}{ResNet-101}
 & BP                     & ---            & \textbf{76.87} & 42.3\,GB \\
\cmidrule(l){2-5}
 & \textsc{priv}, $K{=}4$ & $[3,6,11,13]$  & 76.61          & 16.6\,GB ($-61\%$) \\
 & SOLO, $K{=}4$          & $[3,6,11,13]$  & 76.56          & \textbf{15.6}\,GB ($-63\%$) \\
\bottomrule
\end{tabular}
\end{table}

\begin{table}[t]
\centering
\caption{Language points of Figure~\ref{fig:readout-ladder}b: the
random-readout penalty,
$(\mathrm{ppl}_{\textsc{rand}}/\mathrm{ppl}_{\textsc{priv}}-1)\times100\%$,
against the number of output classes per unit width, $V/d$. (a) Width sweep
at fixed vocabulary, the runs of Table~\ref{tab:lm-ladder-scale}. (b) Vocabulary
sweep at fixed width; for this sweep we report only the penalty. The two
sweeps meet at $V/d{=}64$ ($d{=}512$, $V{=}32{,}768$), where they give 29.80
and 29.72\%, a difference within run-to-run noise.}
\label{tab:lm-readout-penalty}
\footnotesize
\setlength{\tabcolsep}{4pt}
\begin{minipage}[t]{0.60\linewidth}
\centering
\begin{tabular}{@{}rrrrrr@{}}
\toprule
\multicolumn{6}{c}{(a) Width sweep, $V{=}32{,}768$}\\
\midrule
 & & & \multicolumn{2}{c}{Perplexity} & \\
\cmidrule(lr){4-5}
Size & $d$ & $V/d$ & \textsc{rand} & \textsc{priv} & Penalty (\%)\\
\midrule
40M  & 384  & 85.3 & 98.88 & 75.34 & 31.25\\
64M  & 512  & 64.0 & 81.63 & 62.89 & 29.80\\
128M & 768  & 42.7 & 63.76 & 51.10 & 24.77\\
256M & 1024 & 32.0 & 55.34 & 44.57 & 24.16\\
340M & 1024 & 32.0 & 54.12 & 43.82 & 23.51\\
512M & 1280 & 25.6 & 48.99 & 40.24 & 21.74\\
\bottomrule
\end{tabular}
\end{minipage}\hfill
\begin{minipage}[t]{0.36\linewidth}
\centering
\begin{tabular}{@{}rrr@{}}
\toprule
\multicolumn{3}{c}{(b) Vocabulary sweep, $d{=}512$}\\
\midrule
 & & \\
$V$ & $V/d$ & Penalty (\%)\\
\midrule
8{,}192   & 16  & 19.95\\
32{,}768  & 64  & 29.72\\
65{,}536  & 128 & 32.05\\
131{,}072 & 256 & 33.16\\
\bottomrule
\end{tabular}
\end{minipage}
\end{table}

\paragraph{Head depth.}
Table~\ref{tab:hdgrid} and Figure~\ref{fig:hdgrid} vary the depth $H$ of the auxiliary head on
the sweep of Table~\ref{tab:lm-ladder-scale} at three sizes, for SOLO and \textsc{priv}; $H{=}0$
is a linear readout and $H{=}8$, run at 256M only, an over-deep control. Section~\ref{sec:ablations} reads the gap to BP and the advantage of sharing off
these numbers.

\begin{table}[H]
\caption{Auxiliary-head depth on the WikiText-103 sweep, validation perplexity (word-level,
$V{=}32{,}768$, one epoch, $K{=}4$, seed 42) with $H$ blocks per head; $H{=}8$ is run at 256M
only, as an over-deep control. The last column is the end-to-end BP reference of
Table~\ref{tab:lm-ladder-scale}. This grid re-reads the shared copy every step, whereas the sweep
of Table~\ref{tab:lm-ladder-scale} refreshed it every 100 steps, so the $H{=}2$ cells differ
slightly. Bold: best local variant per row pair.}
\label{tab:hdgrid}
\centering
\footnotesize
\setlength{\tabcolsep}{6pt}
\begin{tabular}{l l cccccc c}
\toprule
\textbf{Scale} & \textbf{Variant} & $H{=}0$ & $H{=}1$ & $H{=}2$ & $H{=}3$ & $H{=}5$ & $H{=}8$ & BP \\
\midrule
\multirow{2}{*}{40M, 8L, $d{=}384$}
 & SOLO          & \textbf{76.13} & \textbf{72.06} & \textbf{70.05} & \textbf{69.58} & \textbf{70.10} & --- & \multirow{2}{*}{67.08} \\
 & \textsc{priv} & 81.75 & 76.48 & 75.36 & 73.77 & 73.91 & --- & \\
\midrule
\multirow{2}{*}{128M, 16L, $d{=}768$}
 & SOLO          & \textbf{49.93} & \textbf{48.31} & \textbf{47.38} & \textbf{46.75} & \textbf{46.22} & --- & \multirow{2}{*}{44.46} \\
 & \textsc{priv} & 52.98 & 51.87 & 51.06 & 49.90 & 49.44 & --- & \\
\midrule
\multirow{2}{*}{256M, 16L, $d{=}1024$}
 & SOLO          & \textbf{43.95} & \textbf{42.90} & \textbf{41.95} & \textbf{41.56} & \textbf{41.18} & \textbf{41.48} & \multirow{2}{*}{40.47} \\
 & \textsc{priv} & 46.94 & 45.54 & 44.54 & 44.10 & 43.87 & 44.12 & \\
\bottomrule
\end{tabular}
\end{table}

\begin{figure}[H]
\centering
\includegraphics[width=\textwidth]{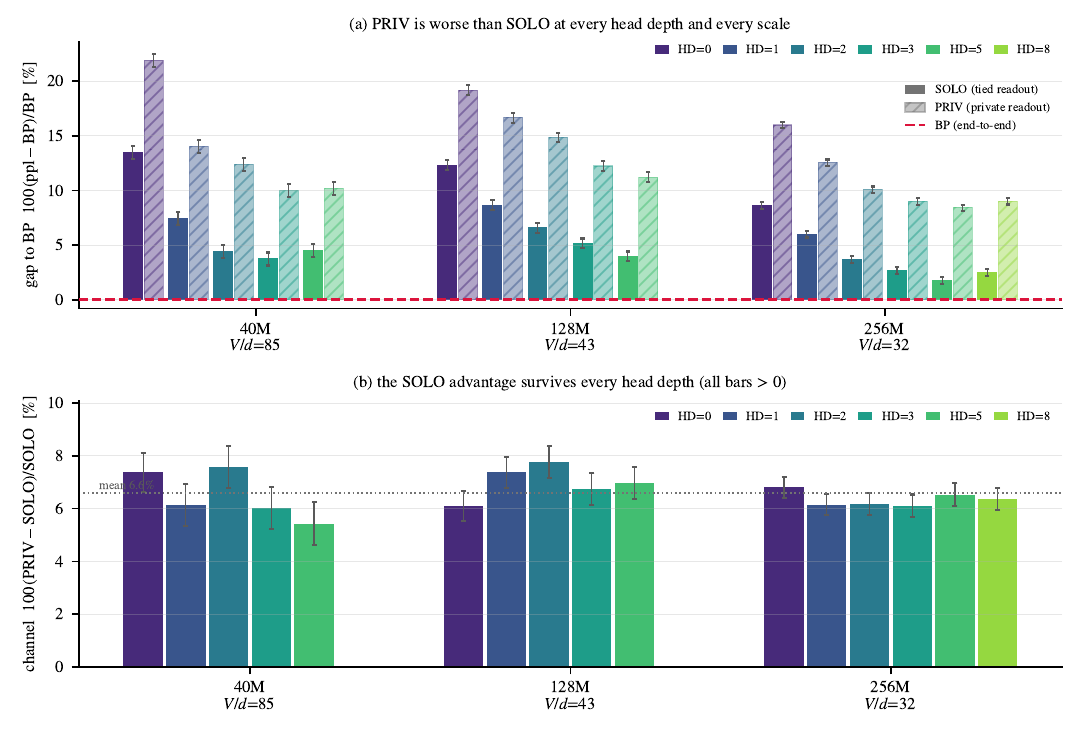}
\caption{Auxiliary-head depth on the WikiText-103 sweep (Table~\ref{tab:hdgrid}). (a) Gap to
BP, $100(\mathrm{ppl}-\mathrm{BP})/\mathrm{BP}$, for SOLO (solid) and \textsc{priv} (hatched) at
$H\in\{0,1,2,3,5\}$, and $8$ at 256M. (b) Advantage of the shared readout,
$100(\textsc{priv}-\mathrm{SOLO})/\mathrm{SOLO}$; the dotted line is the mean over the sixteen
cells, 6.6\%. Bars are seed-42 runs; error bars are the seed-to-seed standard deviation estimated
from paired second-seed runs, 0.40 ppl at 40M and 0.12 at 256M, and inferred at 128M.}
\label{fig:hdgrid}
\end{figure}

\paragraph{Content and basis of the readout.}
Table~\ref{tab:readout-arms} and Figure~\ref{fig:readout-arms} separate what the shared readout provides. \textsc{rot} multiplies the terminal readout by a fixed
random orthogonal matrix $R_k$ per module. The rotation keeps the content
and the range of the readout and changes only its basis. It costs 13.3
perplexity at $H{=}0$ and 7.4 at $H{=}2$ over SOLO. \textsc{rshare} shares
one frozen random matrix across heads and beats \textsc{rand}, one matrix
per head, by 12.0 and 4.2. Each effect exceeds the paired-seed standard
deviation of 0.40 by more than ten times. The gap between \textsc{priv} and
SOLO stays at 5.5 and 5.3, so the advantage of sharing in
Table~\ref{tab:hdgrid} does not depend on head depth.

The parameter-space cosine between the local and the end-to-end gradient
orders the variants as the perplexities do, and its pattern across modules
shows why the basis matters. Each module adds its update to its input, so
its output keeps the basis of the modules before it. The first module also
learns the token embedding and can set its own basis. Under \textsc{rot} at
$H{=}0$, its local gradient aligns with the end-to-end gradient as well as
under SOLO (0.80 for both). Modules 2 and 3 inherit a basis that their heads
do not share, and their alignment falls to 0.67 and 0.61, against 0.78 and
0.84 under SOLO. Two head blocks learn part of the translation between bases
and halve both the perplexity cost and the alignment gap on the last module,
from 0.23 to 0.11. The heads cannot undo $R_k$ exactly, since each ends in a skip connection and an RMSNorm with an element-wise gain, and neither
commutes with $R_k$.

\begin{table}[H]
\caption{Content and basis of the auxiliary readout on the 40M WikiText-103 sweep (8 layers,
$d{=}384$, $V{=}32{,}768$, one epoch, $K{=}4$, copy refreshed every step; seed 42, paired-seed
$\sigma{=}0.40$ ppl). Terminal validation perplexity with $H{=}0$ and $H{=}2$ head blocks, and the
parameter-space cosine between the local and the end-to-end gradient of modules 1 to 3 at the end
of training. \textsc{rot} uses the terminal readout in a fixed random orthogonal basis $R_k$ per
module, so its range equals SOLO's; \textsc{rshare} freezes one random matrix shared by all heads,
\textsc{rand} one per head. $^{\dagger}$\textsc{priv} at $H{=}2$ is the matched run of
Table~\ref{tab:hdgrid}; the two scripts reproduce each other to within 0.04 ppl.}
\label{tab:readout-arms}
\centering
\footnotesize
\setlength{\tabcolsep}{5pt}
\begin{tabular}{l l cc cc}
\toprule
 & & \multicolumn{2}{c}{\textbf{Perplexity} $\downarrow$} & \multicolumn{2}{c}{\textbf{Parameter-space cosine} $\uparrow$} \\
\cmidrule(lr){3-4} \cmidrule(lr){5-6}
\textbf{Arm} & \textbf{Readout of module $k<K$} & $H{=}0$ & $H{=}2$ & $H{=}0$ & $H{=}2$ \\
\midrule
\rowcolor{refrow}
BP & end-to-end, no auxiliary readout & \multicolumn{2}{c}{67.08} & --- & --- \\
SOLO           & $\mathrm{sg}(\bar W)$, shared, live copy   & \textbf{76.16} & \textbf{70.08} & .80 / .78 / .84 & .68 / .80 / .90 \\
\textsc{priv}  & $W_k$, private, learned                          & 81.71 & 75.36$^{\dagger}$ & .75 / .64 / .57 & --- \\
\textsc{rot}   & $\mathrm{sg}(\bar W)R_k$, same content, per-module basis & 89.41 & 77.44 & .80 / .67 / .61 & .63 / .74 / .79 \\
\textsc{rshare} & $W_{\mathrm{rand}}$, frozen, shared by all heads & 104.36 & 94.62 & .50 / .31 / .28 & .46 / .48 / .37 \\
\textsc{rand}  & $W_{\mathrm{rand},k}$, frozen, one per head       & 116.36 & 98.79 & .49 / .33 / .30 & .45 / .47 / .36 \\
\bottomrule
\end{tabular}
\end{table}

\begin{figure}[H]
\centering
\includegraphics[width=0.7\textwidth]{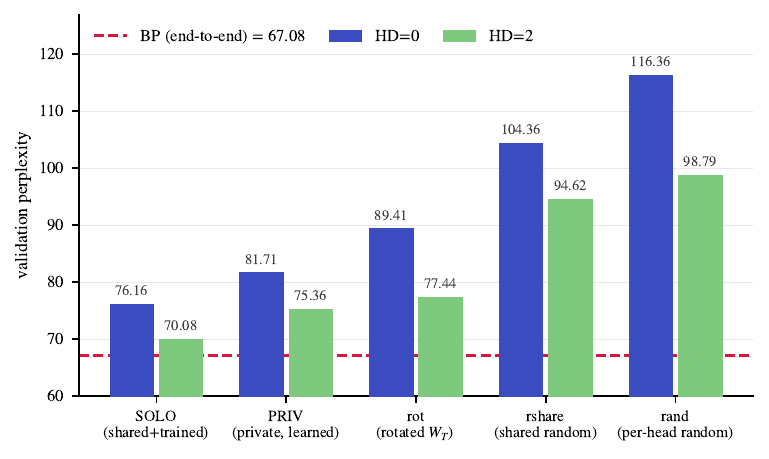}
\caption{The five variants of Table~\ref{tab:readout-arms} at $H{=}0$ and $H{=}2$, terminal validation
perplexity on the 40M sweep; BP dashed. A per-module basis of the same readout (\textsc{rot})
costs 13 ppl at $H{=}0$ and 7 at $H{=}2$; one shared random matrix (\textsc{rshare}) beats
independent ones (\textsc{rand}) by 12 and 4.}
\label{fig:readout-arms}
\end{figure}

\section{Every depth of a SOLO model decodes}
\label{app:exits}

Table~\ref{tab:exit-vs-lens} sets the auxiliary exits of the 1.3B models against the logit lens
applied to the BP model at the same depths. A SOLO exit decodes through the head and shared readout
it was trained with; the lens applies the BP terminal norm and unembedding to an intermediate
state that was never trained to be read.

\begin{table}[h]
\caption{Exits against the logit lens at 1.3B, ppl by depth, mean over three seeds; 24L is the
terminal. On the battery
(geometric-mean ppl over 27 prompts) every SOLO exit is within 2 ppl of its terminal and the BP
lens is not; SOLO exits match the BP terminal's top prediction 71 to 75\% of the time, the lens
16 to 51\%.}
\label{tab:exit-vs-lens}
\centering
\footnotesize
\setlength{\tabcolsep}{5pt}
\begin{tabular}{l cccc cccc}
\toprule
 & \multicolumn{4}{c}{\textbf{WikiText ppl}} & \multicolumn{4}{c}{\textbf{Battery ppl}} \\
\cmidrule(lr){2-5} \cmidrule(lr){6-9}
\textbf{Variants} & 8L & 14L & 20L & 24L & 8L & 14L & 20L & 24L \\
\midrule
SOLO, $K{=}4$ & 28.20 & 24.83 & 24.17 & 23.59 & 29.4 & 28.4 & 29.0 & 27.8 \\
SOLO, $K{=}2$ & ---   & 24.16 & ---   & 22.29 & ---  & 22.6 & ---  & 21.7 \\
BP logit lens & ---   & ---   & ---   & 21.61 & 1347 & 290  & 45.9 & 19.5 \\
\bottomrule
\end{tabular}
\end{table}

\paragraph{Exits at the token level.}
Figure~\ref{fig:token-heat} shows what the exits of Table~\ref{tab:exit-vs-lens} do on one
passage, word by word. The 6-layer exit already reads most of the passage; each deeper exit
changes a handful of words, most often the ones that depend on an earlier mention, and the
terminal differs from the end-to-end model on a few positions rather than everywhere.

\begin{figure}[!htb]
\centering
\includegraphics[width=\textwidth]{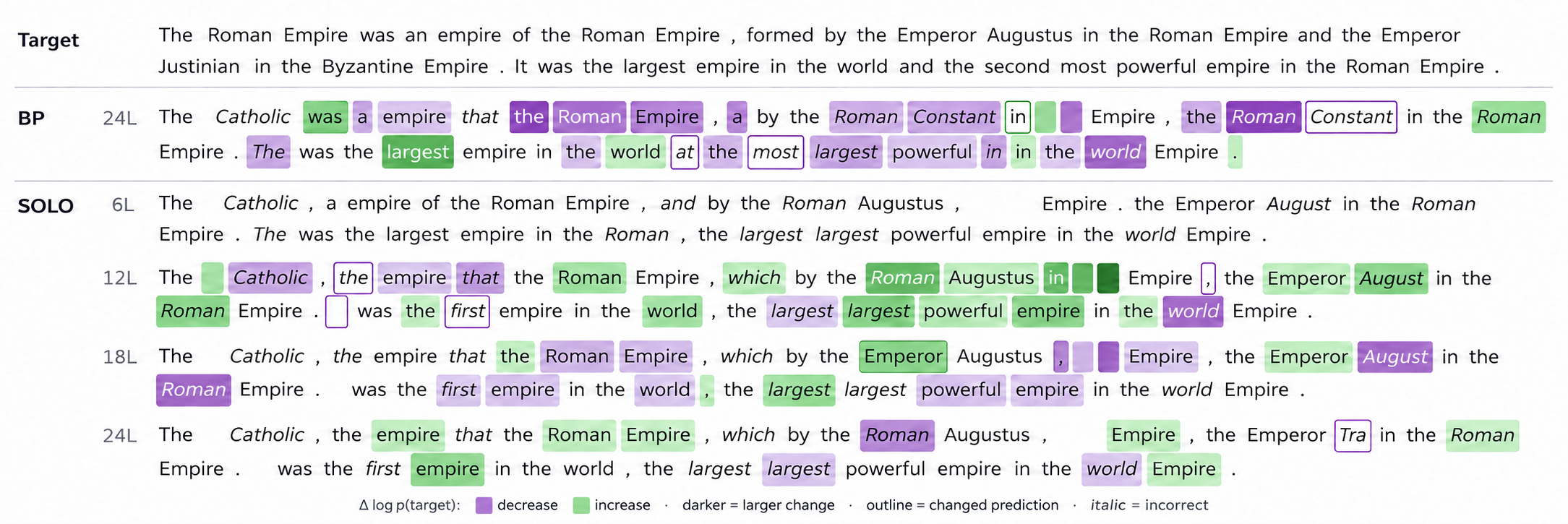}
\caption{One passage decoded word by word by every exit of the 1.3B $K{=}4$ model. A row gives
one readout's next-word prediction, the word itself when correct and the prediction in italics
when not; shading is the change in $\log p$ of the true word from the previous exit (legend in the
figure). Rows are labeled by model depth; 6, 12, and 18 layers are the 8L, 14L, and 20L exits of
Table~\ref{tab:exit-vs-lens}.}
\label{fig:token-heat}
\end{figure}

\paragraph{Terminal-readout probes on the sweep.}
Figure~\ref{fig:lens} asks whether the shared readout reaches the module outputs themselves
rather than only the outputs of the auxiliary heads. On the 40M sweep of
Figure~\ref{fig:align}, the output of every module under each variant is decoded with that variant's own terminal readout, bypassing the auxiliary heads, and compared with BP's logit
lens at the matched depth. Under SOLO and under the under PRETRAINED (SOLO) the intermediate
states are already decodable by the terminal readout, with top-1 agreement with BP of about
0.6 to 0.75 from the from the first module on, whereas under \textsc{priv}, \textsc{rand}, and the
PRETRAINED (BP) agreement stays near 0.1 to 0.25 until the final module, where all variants meet. The module outputs, not only the head outputs, are thus written in a form the
terminal readout reads.

\begin{figure}[!htb]
\centering
\includegraphics[width=\textwidth]{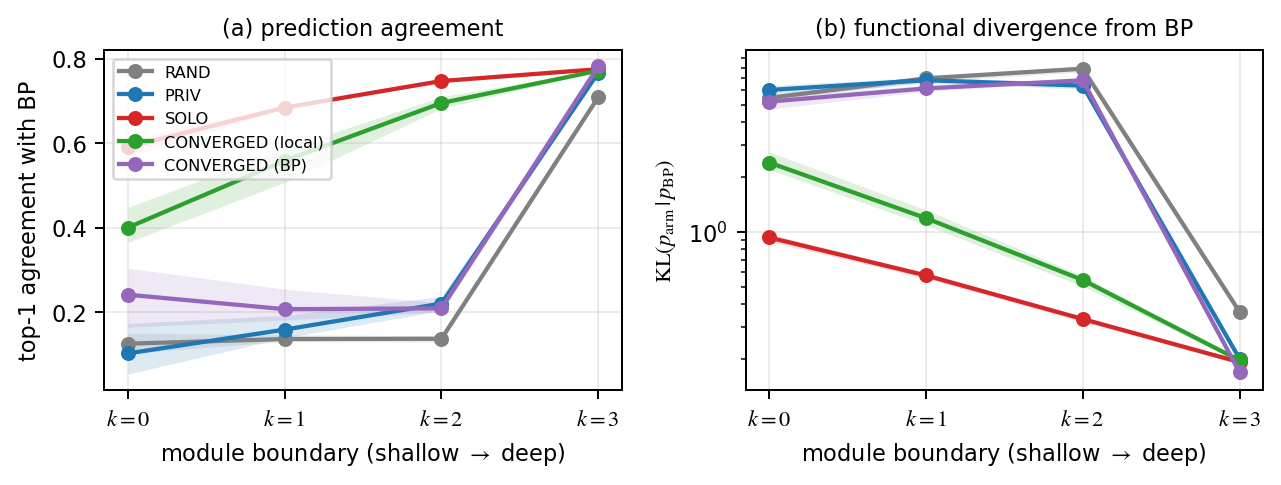}
\caption{Terminal-readout probes across depth on the 40M WikiText-103 sweep. The output of each module is decoded with the model's own terminal readout and compared with
BP's logit-lens distribution at the matched depth. (a) Top-1 prediction agreement.
(b) $\mathrm{KL}(p\,\|\,p_{\mathrm{BP}})$ on a logarithmic scale. Means over three
seeds; bands are the minimum and maximum across seeds. SOLO and \textsc{pretrained} (local)
agree more and diverge less at intermediate boundaries than \textsc{priv}, \textsc{rand}, and
\textsc{pretrained} (BP); the differences narrow at the terminal modules ($k{=}3$). The probe
bypasses the auxiliary heads, so it measures terminal-readout decoding of intermediate
representations, not the predictions of the trained auxiliary exits.}
\label{fig:lens}
\end{figure}

\section{Refresh period of the shared readout}
\label{app:sync}
\begin{table}[!htb]
\caption{Refresh period $S$ of the shared readout copy (WikiText-103 recipe of
Table~\ref{tab:wt103-pubrecipe-sp}, $K{=}4$, one epoch of 7242 steps; mean over three seeds). Subword
perplexity with word-level in parentheses; exits listed shallow to deep; ens is the exit
ensemble. $\Delta$ is the change in validation perplexity relative to $S{=}1$.}
\label{tab:sync}
\centering
\footnotesize
\setlength{\tabcolsep}{4pt}
\begin{tabular}{rr c c c c r}
\toprule
$S$ & refreshes/epoch & val ppl & exits, shallow to deep & ens & test ppl & $\Delta$ val \\
\midrule
1   & 7242 & 65.93 (129.0) & 71.14 / 68.12 / 67.43 / 65.93 & 66.70 & 67.55 & --- \\
10  & 724  & \textbf{65.69} (128.5) & 71.00 / 67.88 / 67.17 / 65.69 & \textbf{66.48} & \textbf{67.33} & $-0.24$ ($-0.4\%$) \\
50  & 145  & 66.16 (129.5) & 71.44 / 68.26 / 67.59 / 66.16 & 66.92 & 67.73 & $+0.23$ ($+0.3\%$) \\
100 & 72   & 67.16 (131.8) & 72.19 / 69.21 / 68.59 / 67.16 & 67.91 & 68.71 & $+1.23$ ($+1.9\%$) \\
200 & 36   & 67.93 (133.6) & 73.04 / 69.99 / 69.35 / 67.93 & 68.67 & 69.48 & $+2.00$ ($+3.0\%$) \\
\bottomrule
\end{tabular}
\end{table}

With $S{=}1$ the copy is re-read after every terminal update; with $S{>}1$ it is frozen and
refreshed every $S$ optimizer steps, the form whose throughput Section~\ref{sec:resources}
measures across devices.
Table~\ref{tab:sync} varies $S$ on the WikiText-103 recipe of
Table~\ref{tab:wt103-pubrecipe-sp} with $K{=}4$ and two-block heads over one epoch of 7242
steps, three seeds per setting. Up to $S{=}50$ the terminal stays within 0.3 perplexity of
per-step reading, and $S{=}10$ is marginally the best setting; $S{=}100$ costs 1.9\% and
$S{=}200$ 3.0\%, with every exit moving together. The cost of a period is set by how far the readout moves within it. Figure~\ref{fig:drift} tracks the relative change
$r_S(t)=\|W(t)-W(t-S)\|_F/\|W(t-S)\|_F$ along the same runs. After warmup the readout moves
by about 0.1\% of its norm per step, 1\% over 10 steps, 3\% over 50, 5\% over 100, and 10\%
over 200, and every curve falls with the cosine schedule. Set against Table~\ref{tab:sync}, a
copy is harmless while it lags the readout by up to about 3\% of its norm and costs 1.9\% at
5\%. The pretraining runs of Table~\ref{tab:slim-untie} re-read the copy at every step
($S{=}1$) and pay none of this cost; the periods measured across devices in
Section~\ref{sec:resources}, 50 steps and below, lie in the harmless range. In
Figure~\ref{fig:align} the alignment of the live copy coincides with that of a the pretrained
copy after the first two thousand steps, so the tolerance to a delayed copy should grow rather
than shrink over a longer run.

\begin{figure}[!htb]
\centering
\includegraphics[width=0.75\textwidth]{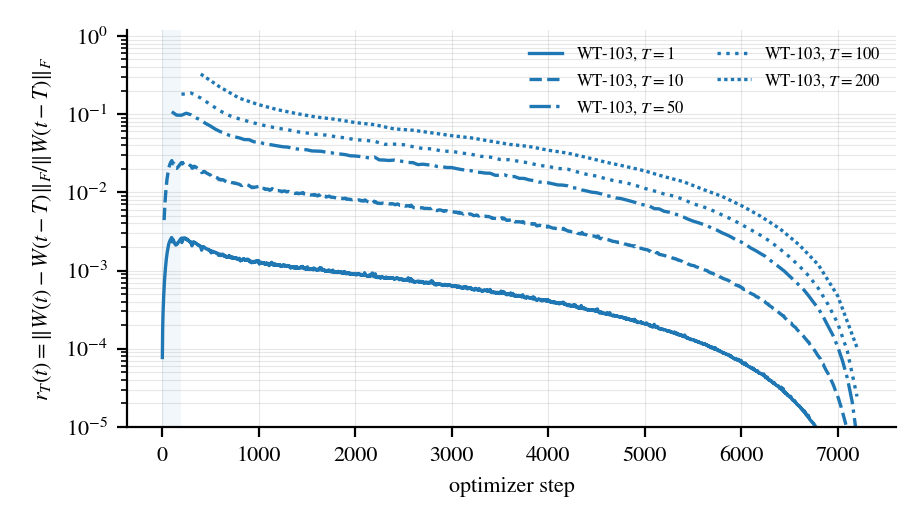}
\caption{Relative change of the terminal readout over one refresh period,
$r_S(t)=\|W(t)-W(t-S)\|_F/\|W(t-S)\|_F$, along the WikiText-103 runs of
Table~\ref{tab:sync} for $S\in\{1,10,50,100,200\}$. The shaded band is the learning-rate
warmup; the fall at the end follows the cosine decay.}
\label{fig:drift}
\end{figure}

\section{Systems measurements}
\label{app:systems}

\begin{figure}[t]
\centering
\includegraphics[width=0.8\textwidth]{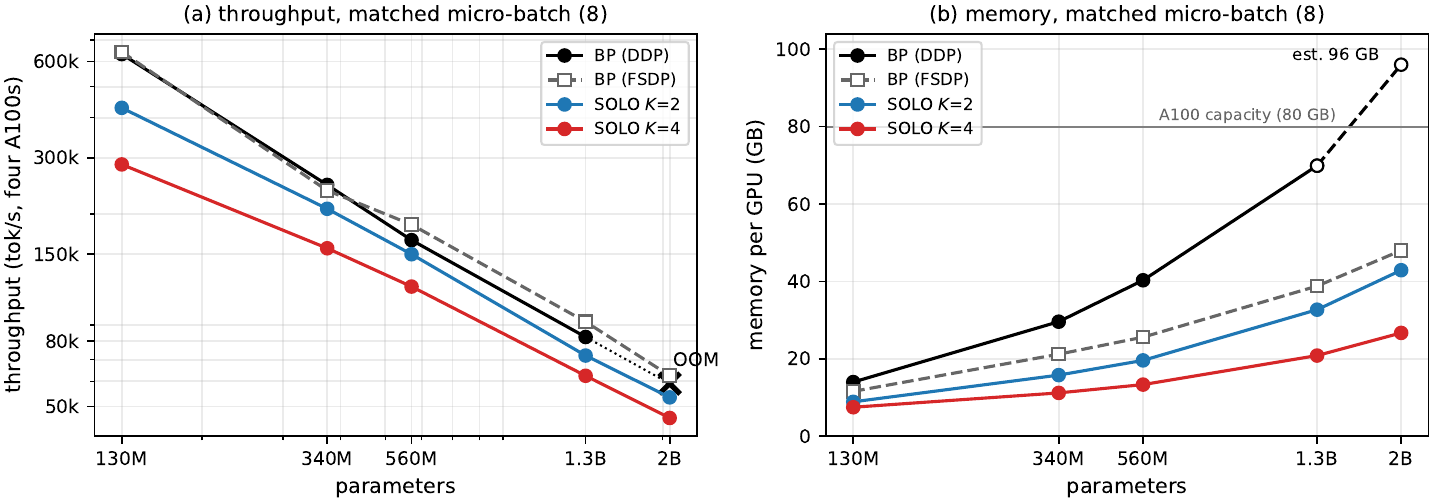}
\caption{Throughput (a) and per-GPU memory (b) against model size on four A100s at matched
micro-batch 8, for replicated (DDP) and sharded (FSDP) backpropagation and the SOLO pipelines.
DDP does not fit at 2B; its memory there is extrapolated (dashed).}
\label{fig:sys-scale}
\end{figure}

\paragraph{Against data-parallel backpropagation.}
This appendix compares SOLO and BP when both use data parallelism. BP is
replicated (DDP) at 340M and 1.3B and sharded (FSDP) at 2B, as in
Table~\ref{tab:slim-untie}, and we also measure FSDP at the smaller sizes.
SOLO is replicated at every size. Figure~\ref{fig:sys-scale} follows the four
configurations from 130M to 2B parameters on four A100s at a matched
micro-batch. Against DDP, the throughput ratio of SOLO rises with width, to
0.88 at $K{=}2$ and 0.76 at $K{=}4$ by 1.3B, because the auxiliary heads are
a fixed cost per module that shrinks relative to the model. At 2B,
replicated DDP no longer fits an 80\,GB device. FSDP is the stronger
baseline. It overtakes DDP from 560M and fits 2B in 48\,GB at 62k tokens per
second. Against FSDP, SOLO runs at 0.78 to 0.88 of the throughput from 340M
up with $K{=}2$ (0.86 at 2B) and at 0.64 to 0.74 with $K{=}4$. SOLO still
uses 11 to 25\% less memory at $K{=}2$ and 35 to 48\% less at $K{=}4$,
because each module releases its activations after its local backward pass.

\paragraph{Setup.}
Figure~\ref{fig:sys-scale} measures five Transformer++ shapes, from 130M parameters
(12 layers, width 768) to 2B (24 layers, width 2560), on four A100s with fused execution,
micro-batch 8 per device for DDP and per pipeline stage, an effective batch of 524{,}288
tokens, and marginal throughput over steps 50 to 150. The throughput ratio rises from
about 0.68 to 0.90 at $K{=}2$ and from about 0.45 to 0.76 at $K{=}4$. At 2B replicated
DDP allocates 75.6\,GB, requests a further 1.95\,GB, and fails on an 80\,GB device; the
estimated 96\,GB requirement is the reference for the 2B memory savings, and the pipeline
trains in 27 to 43\,GB. With each configuration at its own memory-optimal micro-batch, DDP fits 2B
at micro-batch 2 and runs at 46.6k tokens per second in 55.2\,GB; the $K{=}2$ pipeline
runs at 53.4k in 42.9\,GB and the $K{=}4$ pipeline at 44.9k in 26.7\,GB. Table~\ref{tab:slim-untie} reports the sweep values at 340M and 2B and the training-run
values at 1.3B. The FSDP configuration uses PyTorch fully sharded data parallelism with
full sharding and per-block wrapping, bf16 all-gather with fp32 reduction, compiled
execution, and the same optimizer, schedule, micro-batch, effective batch, and 150-step probe
as every other configuration; a same-day DDP rerun at 560M reproduced the earlier sweep value to within
1\%. Memory for DDP and FSDP is the rank-0 peak of a training step, which is symmetric across
ranks; memory for the pipelines is the four-GPU peak average.

\paragraph{Two baselines at 1.3B.}
At 1.3B, Table~\ref{tab:slim-untie} reports the training runs, whose DDP baseline used
micro-batch 4, 76k tokens per second in 48.9\,GB; against it the $K{=}2$ pipeline (77k, 32.7\,GB)
is at parity in throughput with 33\% less memory. At micro-batch 8, DDP runs at 82.5k tokens per
second in 69.9\,GB; against it the $K{=}2$ pipeline saves 53\% of memory at 0.88 of the
throughput and the $K{=}4$ pipeline 70\% at 0.76. FSDP at
1.3B runs at 92.1k tokens per second in 38.8\,GB; against it the $K{=}2$ pipeline saves 16\%
of memory at 0.78 of the throughput and the $K{=}4$ pipeline 46\% at 0.68. All describe the
same pipelines against different baselines. At 340M the two setup
agree, 15.8 against 29.6\,GB at 0.84 of DDP's throughput.

\paragraph{Interconnect.}
On NVLink, DDP's gradient all-reduce costs under 1\% of a step and the pipeline pays for
its auxiliary heads with nothing to win back. Table~\ref{tab:tax-link} removes the
RDMA-class link: NCCL is forced onto its socket transport, the path it takes on any network
without RDMA, over the loopback interface of one node. The model, GPUs per configuration, execution,
and batch are those of the 340M runs. Leaving NVLink costs DDP 72\% of its throughput and
the pipeline 8\%, so the ratio inverts to 2.2, because DDP's per-step all-reduce moves the
full gradient in one synchronized burst of latency-sensitive exchanges, whereas the
pipeline's traffic is a few large activation transfers per micro-batch overlapped with
compute. This is a single-node emulation; the same-split comparison against 1F1B under a
slowed link is Table~\ref{tab:link}.

\begin{table}[h]
\centering
\small
\caption{Throughput off RDMA-class interconnect (340M, four GPUs per arm), a single-node
emulation. Ratio is pipeline over DDP; busy is the fraction of wall-clock the pipeline stages
spend computing. The NVLink row is this session's
own reference and differs from the sweep of Table~\ref{tab:slim-untie} (247k, 156k) by 1 to 4\%.}
\label{tab:tax-link}
\begin{tabular}{lcccc}
\toprule
Interconnect & BP-DDP (tok/s) & SOLO $K{=}4$ (tok/s) & ratio & pipeline busy \\
\midrule
NVLink (reference)         & 250k  & 163k   & 0.65 & 0.96 \\
socket transport, unshaped & 69.2k & 150.1k & 2.2  & 0.92 \\
\bottomrule
\end{tabular}
\end{table}

\paragraph{What is not measured.}
FSDP is measured on NVLink only. Off NVLink it does not change the comparison, since sharding
moves a parameter-sized all-gather and reduce-scatter across the link every step, at least
as much traffic as DDP's all-reduce, so the inversion of Table~\ref{tab:tax-link} would only
grow; we did not run it. Pipelined backpropagation, which trades the 2B memory cliff for
1F1B bubbles, is compared with SOLO under the same split at two, six, and eight stages in
Appendix~\ref{app:pipeline-bp}, on a different model shape from the one measured here.

\paragraph{FLOPs and utilization.}
Table~\ref{tab:flops} counts training FLOPs per token for BP and for SOLO, separating the model layers from the auxiliary blocks, terminal included, and Table~\ref{tab:mfu} converts the measured
throughput into model FLOPs utilization (MFU). The speed comparison is against DDP and FSDP as
PyTorch provides them, so the denominator is plain BP rather than an implementation that hides
communication behind compute.

\begin{table}[t]
\centering
\caption{Training FLOPs per token. One forward pass counts as one unit, so a block costs $3b$
with $b=12d^{2}+2Td$ multiply-accumulate operations per token, the terminal readout costs $3r$
with $r=Vd$, and an auxiliary readout costs $2r$, since its weights are a detached copy and no
gradient with respect to them is computed. One multiply-accumulate is two FLOPs. Embedding
lookups are excluded. The SlimPajama rows use context $2048$ and a $32$k vocabulary and assume
two-block auxiliary heads; the last row is the configuration of Table~\ref{tab:pipe-bp}. The
auxiliary blocks, not the readouts, account for most of the difference.}
\label{tab:flops}
\footnotesize
\setlength{\tabcolsep}{4pt}
\begin{tabular}{lrrrrrrrrrr}
\toprule
& & & & & \multicolumn{2}{c}{BP} & \multicolumn{3}{c}{SOLO} & \\
\cmidrule(lr){6-7}\cmidrule(lr){8-10}
model & $L$ & $d$ & $K$ & $H$ & total & layers & total & aux.\ blocks & readouts & $\rho$ \\
\midrule
340M          & 24 & 1024 & 2 & 2 & 2.61  & 2.42  & 2.94  & 0.20 & 0.33 & 1.127 \\
1.3B          & 24 & 2048 & 2 & 2 & 8.85  & 8.46  & 9.82  & 0.70 & 0.66 & 1.109 \\
2B            & 24 & 2560 & 4 & 2 & 13.33 & 12.83 & 17.52 & 3.21 & 1.48 & 1.315 \\
96-block 1.2B & 96 & 1024 & 8 & 1 & 8.51  & 8.46  & 9.36  & 0.62 & 0.29 & 1.100 \\
\bottomrule
\end{tabular}

\vspace{2pt}
{\footnotesize All totals in GFLOPs per token. $\rho$ is defined in \eqref{eq:rho}. It is sensitive to
the head depth and to the number of modules: for the 340M shape it is $1.089$, $1.266$ and
$1.621$ at $K{=}2,4,8$ with $H{=}1$, and $1.127$, $1.382$ and $1.891$ with $H{=}2$.}
\end{table}

\begin{table}[t]
\centering
\caption{Model FLOPs utilization of the measured runs, computed from Table~\ref{tab:flops} and
the reported throughput, with $312$ TFLOPs as the bf16 peak of an A100 and throughput summed over
the devices of a run. The SOLO row uses its own FLOPs per token, so the auxiliary heads count as
work rather than as overhead. The 96-block configuration used for the pipeline comparison runs at
about a third of the utilization of the pretraining runs, because its blocks are narrow and its
module is not compiled; its absolute throughput is therefore not representative and only ratios
within a row of Table~\ref{tab:pipe-bp} should be read.}
\label{tab:mfu}
\small
\begin{tabular}{llrrrr}
\toprule
run & parallelism & GPUs & tok/s & GFLOPs/token & MFU \\
\midrule
1.3B, micro-batch 8 & BP, replicated DP & 4 & 82{,}500 & 8.85  & 58.5\% \\
1.3B                & BP, sharded DP    & 4 & 92{,}100 & 8.85  & 65.3\% \\
2B                  & BP, sharded DP    & 4 & 62{,}000 & 13.33 & 66.2\% \\
\midrule
96-block 1.2B       & BP, replicated DP & 8 & 65{,}471 & 8.51  & 22.3\% \\
96-block 1.2B       & BP, 1F1B pipeline & 8 & 59{,}441 & 8.51  & 20.3\% \\
96-block 1.2B       & SOLO, pipeline    & 8 & 59{,}770 & 9.36  & 22.4\% \\
\bottomrule
\end{tabular}
\end{table}

\section{Comparison with pipeline-parallel backpropagation}
\label{app:pipeline-bp}

Appendix~\ref{app:systems} compares SOLO with backpropagation under data parallelism. Update
locking costs most in pipeline parallelism, where each stage waits for the gradient of the next.
This appendix therefore gives both methods the same split into the same number of stages, so
that the only difference is whether a stage waits for that gradient.

\paragraph{Setup.}
A model of $L{=}96$ blocks with $d{=}1024$, $T{=}1024$ and $V{=}8192$, about $1.2$B parameters, is
split evenly across $p$ stages of one A100-80GB each, on one node with NVLink. Micro-batches hold
$B{=}4$ sequences and an optimizer step consumes $M$ of them, so both methods use the same global
batch and the same number of optimizer steps. SOLO places $K{=}p$ modules with one auxiliary block
per head ($H{=}1$) and a shared readout refreshed by a broadcast from the last stage every $S$
steps. Training is bf16 with AdamW on synthetic Zipf-distributed tokens. We report the throughput
of the slowest stage, peak memory from the allocator, and inter-stage traffic computed from tensor
shapes. Throughput is measured over 10 steps after 5 warm-up steps, and over 50 steps for the
configurations whose readout broadcast has to fall inside the window. Repeated measurements of the
same cell agree to within $0.4\%$ for 1F1B and SOLO and $0.7\%$ for recomputation. The BP baselines
are the 1F1B schedule \citep{narayanan2019pipedream}, the interleaved schedule (VPP) with $v$
virtual stages per GPU \citep{narayanan2021efficient}, and 1F1B with activation recomputation of
every block. VPP uses the PyTorch implementation, and so does a second 1F1B baseline, used as a
check on our own.

\paragraph{Cost model.}
Per token, a block costs $b=12d^{2}+2Td$ multiply-accumulate operations and a full-vocabulary
readout costs $r=Vd$. We count one forward pass as one unit, so a block costs $3b$ and the terminal
readout costs $3r$. An auxiliary readout costs $2r$, because its weights are a detached copy of the
terminal readout and no gradient with respect to them is computed. Hence $W_{\mathrm{BP}}=3Lb+3r$
and $W_{\mathrm{SOLO}}=3\bigl(L+(K-1)H\bigr)b+(2K+1)r$.

\begin{proposition}[Work ratio]
\label{prop:rho}
Let $\gamma \triangleq r/(Lb) = V/\bigl(L(12d+2T)\bigr)$ be the cost of one readout relative to
the $L$ blocks. Then
\begin{equation}
\rho \;\triangleq\; \frac{W_{\mathrm{SOLO}}}{W_{\mathrm{BP}}}
\;=\;
\frac{1+\dfrac{(K-1)H}{L}+\dfrac{(2K+1)\gamma}{3}}{1+\gamma},
\qquad
\rho-1 \;=\; (K-1)\,\frac{H/L+2\gamma/3}{1+\gamma}.
\label{eq:rho}
\end{equation}
\end{proposition}

Both terms in $\rho-1$ become small for large models. The readout term scales as
$\gamma\sim V/(Ld)$. The head term $(K-1)H/L$ is approximately $H/(L/p)$, the number of auxiliary
blocks divided by the number of blocks per stage. Width does not appear in the head term,
because model blocks and auxiliary blocks have the same cost at any width. To check the model
without pipeline effects we measured each stage separately with communication turned off and added
the times. At $L{=}24$ with $d$, $T$ and $V$ as above, $K{=}p{=}2$ and micro-batch $4$, the two
stages take $0.1297+0.1317=0.2614$\,s for backpropagation, $0.1330+0.1428=0.2759$\,s for SOLO with
$H{=}1$ and $0.1436+0.1428=0.2864$\,s with $H{=}2$, giving $\rho=1.0552$ and $1.0953$ against
predicted $1.0562$ and $1.0969$, a difference of at most $0.2\%$.

\paragraph{Condition for SOLO to be faster.}
With $M$ micro-batches per step the 1F1B schedule is idle for a fraction $\beta=(p-1)/(M+p-1)$ of
the time, because stage $s$ cannot run the backward pass of one micro-batch until stage $s{+}1$ has
returned the gradient of the previous one. A SOLO stage runs its backward pass without waiting, so
its pipeline has no idle time of this kind. SOLO has higher steady-state throughput when
$\rho<1/(1-\beta)=1+(p-1)/M$. With $p=K$ the factor $K-1$ in \eqref{eq:rho} cancels and the
condition becomes
\begin{equation}
M \;<\; M^{\star} \;=\; \frac{1+\gamma}{H/L+2\gamma/3}
\;\;\xrightarrow[\;\gamma\to 0\;]{}\;\; \frac{L}{H},
\label{eq:win}
\end{equation}
which does not depend on the number of stages. For the configuration above \eqref{eq:win}
gives $M^{\star}=69.9$.

\begin{table}[t]
\centering
\caption{Throughput under the same split, relative to 1F1B at the same $p$ and $M$. Values above
one mean the method is faster than 1F1B. Configuration (A): $L{=}96$, $d{=}1024$, $T{=}1024$,
$V{=}8192$, $H{=}1$, micro-batch $4$, $K{=}p$. A dash means the configuration was not run at that
$p$. The data-parallel runs use the same $p$ devices as full replicas with the same global batch,
and are included as context rather than as a same-split comparison.}
\label{tab:pipe-bp}
\small
\begin{tabular}{lrrrrrrrr}
\toprule
& \multicolumn{2}{c}{$p{=}2$} & \multicolumn{2}{c}{$p{=}4$} & \multicolumn{2}{c}{$p{=}6$} & \multicolumn{2}{c}{$p{=}8$} \\
\cmidrule(lr){2-3}\cmidrule(lr){4-5}\cmidrule(lr){6-7}\cmidrule(lr){8-9}
method & $M{=}24$ & $72$ & $24$ & $72$ & $24$ & $72$ & $24$ & $72$ \\
\midrule
BP, 1F1B                    & 1.000 & 1.000 & 1.000 & 1.000 & 1.000 & 1.000 & 1.000 & 1.000 \\
BP, 1F1B (PyTorch)          & --    & --    & --    & --    & --    & --    & 0.980 & 0.978 \\
BP, VPP-2                   & --    & --    & 1.017 & 0.982 & --    & --    & 1.067 & 0.990 \\
BP, VPP-4                   & --    & --    & 1.040 & 0.956 & --    & --    & 1.062 & 0.859 \\
BP, 1F1B + recomputation    & 0.766 & 0.768 & --    & --    & 0.750 & 0.745 & 0.748 & 0.736 \\
\midrule
SOLO, $S{=}50$              & \textbf{1.018} & 0.992 & \textbf{1.072} & 0.994 & \textbf{1.123} & 0.998 & \textbf{1.179} & 1.006 \\
SOLO, $S{=}10$              & --    & --    & 1.070 & 0.993 & --    & --    & 1.172 & 1.006 \\
SOLO, $S{=}1$               & --    & --    & 1.038 & 0.983 & --    & --    & 1.098 & 0.981 \\
\midrule
BP, data parallel           & 1.032 & 1.011 & --    & --    & 1.175 & 1.068 & 1.246 & 1.097 \\
SOLO, data parallel         & 1.020 & 0.999 & --    & --    & 1.108 & 1.008 & 1.147 & 1.011 \\
\bottomrule
\end{tabular}
\end{table}

\paragraph{Throughput.}
Table~\ref{tab:pipe-bp} and Figure~\ref{fig:pipe-bp} give the measurements. We fit $\rho$ on the
$M{=}72$ column of the SOLO row with $S{=}50$ and predict the $M{=}24$ column, which no fit uses.
The fitted values are $\rho=1.022$, $1.048$, $1.072$ and $1.091$ at $p=2$, $4$, $6$ and $8$; the
predicted ratios at $M{=}24$ are $1.019$, $1.073$, $1.127$ and $1.184$ against measured $1.018$,
$1.072$, $1.123$ and $1.179$. The per-module cost $(\rho-1)/(p-1)$ is $0.022$, $0.016$, $0.014$ and
$0.013$, against $0.0143$ from \eqref{eq:rho}; at $p{=}2$ the single auxiliary module absorbs all
fixed overhead. The measured crossover $M^{\star}=(p-1)/(\rho-1)$ is $45$, $63$, $69$ and $77$,
close to the $p$-independent value $69.9$ of \eqref{eq:win} except at $p{=}2$. As a check on the
schedule itself, the ratio of 1F1B throughput between the two values of $M$ agrees with
$M/(M+p-1)$ to within $0.3\%$ at every $p$, and the PyTorch 1F1B implementation runs at $0.98$ of
ours, so the baseline is not weakened by our implementation.

Measured against the fastest pipeline schedule in each cell rather than against 1F1B alone,
SOLO with $S{=}10$ is $10\%$ faster at $p{=}8$, $M{=}24$ and $0.6\%$ faster at $M{=}72$; with
$S{=}1$ it is $2.9\%$ faster and $1.9\%$ slower. VPP reaches $0.95$ of the throughput its own
bubble model predicts at $v{=}2$ and $0.88$ at $v{=}4$ for $M{=}24$, and $0.80$ at $v{=}4$ for
$M{=}72$, where it sends $69.8$\,GiB per step. An implementation that reached the bubble model
would match SOLO at $M{=}24$, so the margin at small $M$ depends on how efficiently interleaving
is implemented.

\begin{table}[t]
\centering
\caption{Memory and communication at $p{=}8$ and $M{=}24$, configuration (A). Peak memory is given
for the most loaded device and as the mean over the eight devices. Activations are the peak minus
resident state minus gradients. Traffic is computed from tensor shapes; the pipeline runs scale
linearly with $M$ and the data-parallel runs do not. The VPP and PyTorch 1F1B runs send
activations in fp32, which doubles their traffic relative to an implementation that sends bf16.}
\label{tab:pipe-mem}
\footnotesize
\setlength{\tabcolsep}{3pt}
\begin{tabular}{lrrrrr}
\toprule
method & peak GB, max & peak GB, mean & act.\ GB, max & act.\ GB, mean & traffic GiB/step \\
\midrule
BP, 1F1B                 & 19.8 & 12.1 & 17.4 & 9.8  & 2.6 \\
BP, 1F1B (PyTorch)       & 20.1 & 13.2 & 17.3 & 9.7  & 5.3 \\
BP, VPP-2                & 28.8 & 21.6 & 24.8 & 17.0 & 11.3 \\
BP, VPP-4                & 27.0 & --   & 20.9 & --   & 23.3 \\
BP, 1F1B + recomputation & 4.5  & 3.6  & 2.1  & 1.3  & 2.6 \\
SOLO, $S{=}50$           & \textbf{5.1} & \textbf{5.0} & \textbf{2.5} & \textbf{2.5} & \textbf{1.3} \\
\midrule
BP, data parallel        & 40.3 & 40.3 & 17.4 & 17.4 & 64.0 \\
SOLO, data parallel      & 27.1 & 27.1 & 2.5  & 2.5  & 68.6 \\
\bottomrule
\end{tabular}
\end{table}

\paragraph{Large $M$ and a second model shape.}
Figure~\ref{fig:pipe-three} extends Table~\ref{tab:pipe-bp} in two directions. At $p{=}8$ on the
96-block model, the ratio to 1F1B continues past the crossover to $0.966$ at $M{=}128$ and
$0.943$ at $M{=}256$ with the copy refreshed every 50 steps ($0.953$ and $0.936$ every step),
on the curve $(M+p-1)/(\rho M)$ with $\rho$ fitted at $M{=}72$ alone, and approaches $1/\rho=0.92$
from above, so the penalty at large $M$ is bounded by the arithmetic overhead. The 24-block
model of the pretraining runs ($d{=}2048$, $T{=}2048$, $V{=}32$k, $H{=}2$, micro-batch 1) has
12 to 3 blocks per stage as $p$ grows from 2 to 8 and loses with $p$: $0.91$, $0.87$, and
$0.82$ of 1F1B at $M{=}24$ and $0.88$, $0.81$, and $0.72$ at $M{=}72$. The per-module cost
$(\rho-1)/(p-1)$ recovered from these ratios is $0.08$ to $0.15$, against $0.013$ to $0.022$
on the 96-block model, and a one-block head at $p{=}4$ brings the ratio back to $0.98$. The sign
of the same-split comparison is thus set by the number of blocks per stage relative to
the head, as \eqref{eq:win} predicts; on shallow models the memory saving and the micro-batch
lever of Table~\ref{tab:mbsweep} are what remain. Figure~\ref{fig:pipe-three}c measures the
per-module cost directly, under data parallelism, where there is no bubble and $\rho$ is the ratio
of the two throughputs, for $K{=}2$, 4, 8 and $H{=}1$, 2, 4 on the 24-block model at micro-batch 8.
The nine cells fall on one line, $0.010+0.040H$, whose slope equals the head-block term of
\eqref{eq:rho}, $(1/L)/(1+\gamma)=0.040$, and whose intercept is a third of the readout term,
$0.030$, so the extra readouts cost less than their FLOP count. The values recovered from the
pipeline throughput ratios at $p{=}4$ have a slope 29\% higher; the difference is scheduling
and implementation overhead, not arithmetic.

\begin{figure}[!htb]
\centering
\includegraphics[width=\textwidth]{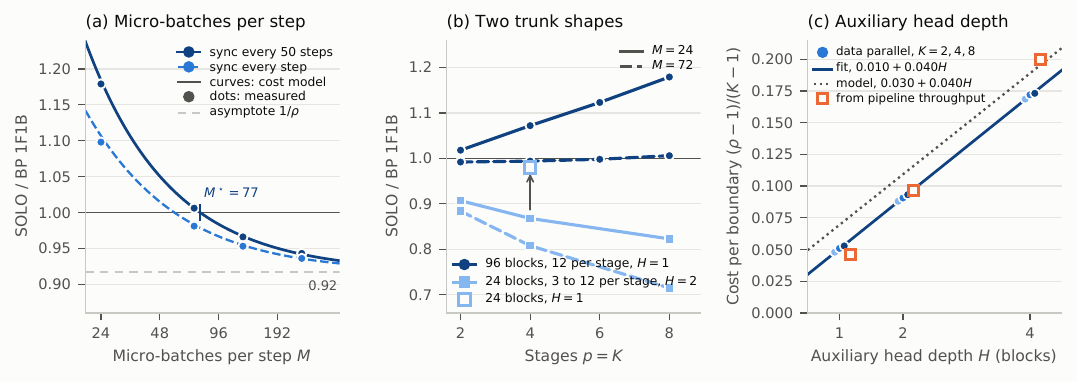}
\caption{Same-split throughput and the per-module cost. (a) SOLO relative to 1F1B against
micro-batches per step at $p{=}8$ on the 96-block model of configuration (A); curves are the cost
model with $\rho$ fitted at $M{=}72$, dots are measured, the dashed curve adds the synchronization
cost of a broadcast at every step, the gray line is $1/\rho$. (b) The same ratio against the number
of stages at $M{=}24$ (solid) and $72$ (dashed), for the 96-block model with $H{=}1$ and for the
24-block model of the pretraining runs ($d{=}2048$, $T{=}2048$, $V{=}32$k, micro-batch 1) with
$H{=}2$; the open marker is the 24-block model with $H{=}1$ at $p{=}4$. (c) Cost per module
$(\rho-1)/(K-1)$ against auxiliary-head depth under data parallelism on the 24-block
model (micro-batch 8, $M{=}24$), for $K{=}2$, 4, 8; solid line the fit, dotted line the FLOP
model of \eqref{eq:rho}, open squares the values recovered from pipeline throughput at $p{=}4$.}
\label{fig:pipe-three}
\end{figure}

\paragraph{Memory.}
The 1F1B schedule keeps $p-s+1$ micro-batches in progress on stage $s=1,\dots,p$, so that stage
stores $(p-s+1)(L/p)a$ activations, where $a$ is the activation of one block for one micro-batch.
The first stage therefore stores $La$, as much as the unsplit model, for any $p$, and the mean over
stages is $\frac{p+1}{2p}La$. A SOLO stage holds one micro-batch, that is $(L/p)a$ plus its
auxiliary head. Table~\ref{tab:pipe-mem} matches this. The activation peak of 1F1B is $17.25$,
$17.29$, $17.33$ and $17.37$\,GB at $p=2$, $4$, $6$ and $8$, and equals the $17.39$\,GB of a single
data-parallel replica; its ratio to SOLO is $1.94$, $3.73$, $5.40$ and $6.95$, against $p$ from the
model, and the ratio of the means at $p{=}8$ is $3.95$ against $(p+1)/2=4.5$. The difference is the
auxiliary head, which also raises the resident state of SOLO by $0.2$\,GB per device. Peak memory
on the most loaded device falls from $26.4$, $22.0$, $20.5$ and $19.8$\,GB to $18.3$, $9.5$, $6.6$
and $5.1$\,GB. Memory does not change with $M$ for either method, except in the PyTorch runs,
which hold one fp32 message per micro-batch.

\paragraph{Recomputation.}
Backpropagation can reach SOLO's activation memory by recomputing. Recomputing every block leaves
only block inputs, which brings the activation peak to $2.1$\,GB, below SOLO, at $0.74$ to $0.77$
of 1F1B throughput across $p$ and $M$, against $3(1+\gamma)/(4+3\gamma)=0.751$ from the cost model.
SOLO is then $1.29$ to $1.57$ times faster than this baseline while using $13$ to $35\%$ more
memory than it. Interpolating between the two backpropagation end points to the peak memory of
SOLO, which assumes that time and memory are linear in the fraction of recomputed blocks, gives
SOLO a factor of $1.19$, $1.47$ and $1.55$ at $M{=}24$ and $1.16$, $1.31$ and $1.35$ at $M{=}72$
for $p=2$, $6$ and $8$. Limiting the number of micro-batches in flight is the other way to save
memory at the cost of utilization, and a worse one: holding $k\le p$ of them gives activation
$k(L/p)a$ and throughput about $k/p$ of the bubble-free rate, so $k{=}1$ reaches the memory of
SOLO at $16\%$ of the throughput of 1F1B at $p{=}8$, $M{=}24$.

\paragraph{Micro-batch sweep.}
Table~\ref{tab:mbsweep} and Figure~\ref{fig:mbsweep} hold the global batch fixed at 96 sequences
per step, 98,304 tokens, at $p{=}8$ and vary the micro-batch from 1 to 32 sequences, so
$M{=}96/B$. Under 1F1B a larger micro-batch multiplies the activations in flight and shrinks
$M$, which widens the bubble $(p-1)/(M+p-1)$ from 23\% at $M{=}24$ to 70\% at $M{=}3$; its
throughput peaks at micro-batch 4, 50.6k tokens per second in 19.8\,GB, and falls to 25.4k at
micro-batch 32. A SOLO stage holds one micro-batch and has no bubble, so its throughput rises
with the micro-batch until the kernels saturate, to 72.9k at micro-batch 16 in 11.6\,GB and
72.4k at micro-batch 32 in 20.3\,GB. At each method's best micro-batch SOLO is $1.44\times$
faster, and at equal peak memory, 20\,GB, $1.43\times$. With the copy refreshed every step the
improvement is smaller, 60.4k at micro-batch 8, since the stages align at every refresh. SOLO's
throughput times the 1F1B bubble factor $M/(M+p-1)$, the dotted line in
Figure~\ref{fig:mbsweep}b, tracks the 1F1B curve to within the auxiliary-head overhead, so the
fall of 1F1B is the bubble and not arithmetic. The activation memory that gradient isolation
frees therefore converts into throughput through the micro-batch at a fixed global batch.

\begin{table}[!htb]
\centering
\caption{Micro-batch sweep at a fixed global batch of 96 sequences per step, configuration (A),
$p{=}8$. Peak memory on the most loaded device; SOLO with the copy refreshed every 50 steps and
every step.}
\label{tab:mbsweep}
\small
\begin{tabular}{rr rr rrr}
\toprule
 & & \multicolumn{2}{c}{BP, 1F1B} & \multicolumn{3}{c}{SOLO} \\
\cmidrule(lr){3-4}\cmidrule(lr){5-7}
$B$ & $M$ & tok/s & peak GB & tok/s, $S{=}50$ & tok/s, $S{=}1$ & peak GB \\
\midrule
1  & 96 & 21{,}980 & 8.4  & 20{,}272 & 20{,}994 & 3.5 \\
2  & 48 & 43{,}125 & 12.3 & 40{,}784 & 41{,}218 & 4.0 \\
4  & 24 & \textbf{50{,}576} & 19.8 & 59{,}333 & 55{,}537 & 5.1 \\
8  & 12 & 48{,}042 & 34.7 & 68{,}545 & \textbf{60{,}361} & 7.3 \\
16 & 6  & 37{,}943 & 48.5 & \textbf{72{,}877} & 57{,}387 & 11.6 \\
32 & 3  & 25{,}374 & 47.9 & 72{,}361 & 47{,}615 & 20.3 \\
\bottomrule
\end{tabular}
\end{table}

\begin{figure}[!htb]
\centering
\includegraphics[width=\textwidth]{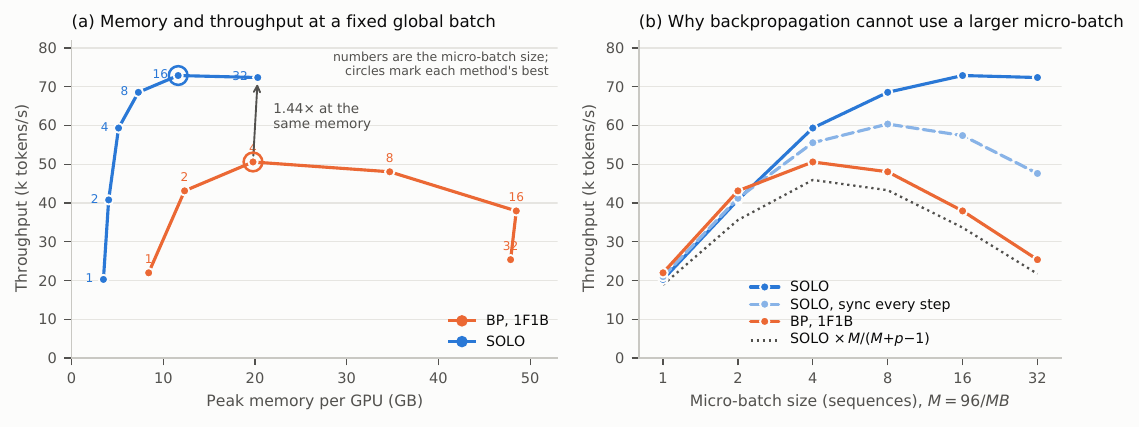}
\caption{Micro-batch sweep of Table~\ref{tab:mbsweep}. (a) Throughput against peak memory per
GPU, points labeled by micro-batch size, circles at each method's best. (b) Throughput against
micro-batch size; the dotted line is SOLO's throughput times the 1F1B bubble factor $M/(M+p-1)$.}
\label{fig:mbsweep}
\end{figure}

\paragraph{Readout synchronization.}
The broadcast sends $(p-1)(Vd+V)$ numbers every $S$ steps, $2.2$\,MiB per step at $S{=}50$ and
$112$\,MiB at $S{=}1$ in this configuration, against $2688$\,MiB of activations. Its cost is not
the transfer but the alignment it forces, since the stages, which otherwise run one forward pass
apart, must reach the same step. The throughput lost relative to $S{=}50$ follows
$0.24\,(p-1)/(M S)$, fitted through the origin on eight points. At $S{=}1$ the measured loss is
$3.2\%$ and $1.1\%$ at $p{=}4$ and $6.9\%$ and $2.4\%$ at $p{=}8$ for $M{=}24$ and $M{=}72$; at
$S{=}10$ it is at most $0.6\%$. SOLO sends less than 1F1B in total when $S\,M\,B\,T>V$. In this
configuration that holds by three orders of magnitude, but for $V{=}128$k, $d{=}4096$, $T{=}4096$,
$B{=}1$, $M{=}32$ and $p{=}8$ the two totals are equal at $S{=}1$, and SOLO sends $0.55$ of 1F1B
at $S{=}10$.

\paragraph{Communication and link speed.}
Across the same $p-1$ links between stages, SOLO sends activations once and 1F1B sends them twice,
so SOLO sends half as much for any $p$, and a quarter to a ninth of VPP at equal precision, which
places $v$ virtual stages per GPU and multiplies traffic by $(vp-1)/(p-1)$ (a ninth to an
eighteenth against the PyTorch implementation, which sends fp32). This comparison applies to
pipeline parallelism only. Data-parallel backpropagation exchanges gradients sized by the parameter
count, $64$\,GiB per step at $p{=}8$ here, independent of $M$. Table~\ref{tab:link} measures the
effect of a slower link in configuration (B). The throughput of backpropagation falls as the link
slows, while SOLO changes by less than $1.2\%$ down to $1$\,Gb/s, where the ratio reaches $2.10$,
the ratio of the bytes sent. Beyond a node, the removed dependency matters more than the halved byte count: a SOLO stage
never waits for a gradient, so its throughput is flat where 1F1B's has halved. Cross-node pipelines, the regime this points to, are not run here.

\begin{table}[t]
\centering
\caption{Throughput for different link speeds, configuration (B): $L{=}24$, $H{=}2$, $M{=}8$,
$p{=}2$, each method using the fastest of the layer splits $11{:}13$, $12{:}12$ and $13{:}11$.
This configuration differs from Table~\ref{tab:pipe-bp} and the two should not be compared
directly. The rows from $10$\,Gb/s down shape the loopback link of a single node with \texttt{tc tbf}. At the two slowest settings the measured transfer rate is $88$ to $92\%$
of the nominal rate.}
\label{tab:link}
\small
\begin{tabular}{lrrrr}
\toprule
link & nominal MB/s & 1F1B & SOLO & SOLO / 1F1B \\
\midrule
NVLink    & n/a  & 57{,}685 & 59{,}606 & 1.033 \\
socket    & n/a  & 55{,}679 & 59{,}363 & 1.066 \\
10\,Gb/s  & 1250 & 52{,}667 & 59{,}318 & 1.126 \\
5\,Gb/s   & 625  & 48{,}401 & 59{,}328 & 1.226 \\
2\,Gb/s   & 250  & 38{,}989 & 59{,}213 & 1.519 \\
1\,Gb/s   & 125  & 28{,}084 & 58{,}901 & \textbf{2.097} \\
0.5\,Gb/s & 62   & 14{,}568 & 30{,}351 & 2.083 \\
\bottomrule
\end{tabular}
\end{table}

\begin{figure}[t]
\centering
\includegraphics[width=\linewidth]{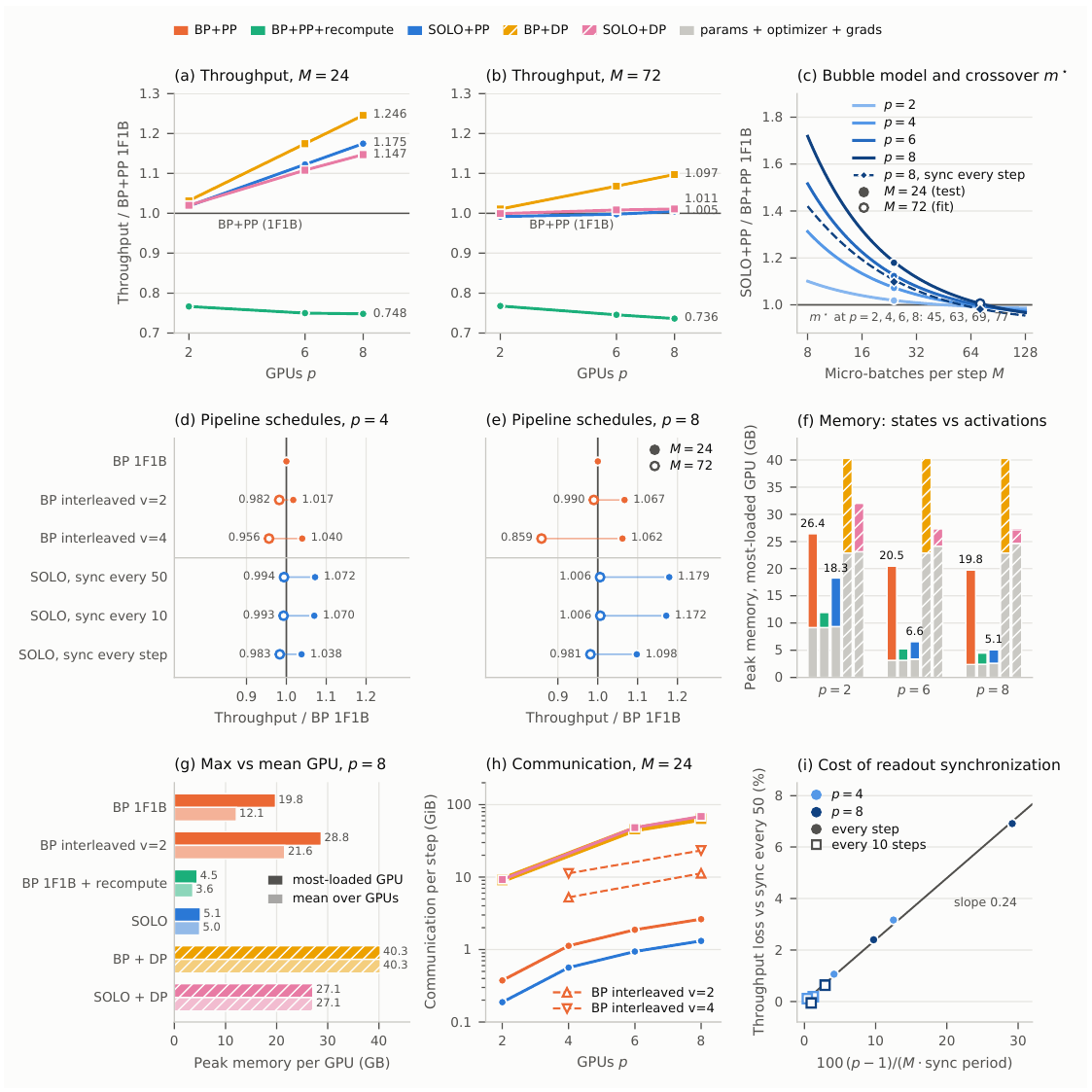}
\caption{Same-split comparison, configuration (A). (a, b) Throughput relative to 1F1B against the
number of stages at $M{=}24$ and $M{=}72$. (c) The bubble model of \eqref{eq:win} against the
number of micro-batches, with $\rho$ fitted at $M{=}72$ and the $M{=}24$ points held out; the
dashed line adds the synchronization cost at $S{=}1$. (d, e) Throughput of each schedule at $p{=}4$
and $p{=}8$. (f) Peak memory on the most loaded device, split into resident state and activations.
(g) Peak memory at $p{=}8$, most loaded device against the mean over devices. (h) Traffic per step.
(i) Throughput lost to readout synchronization against $(p-1)/(M S)$, with the fitted slope.}
\label{fig:pipe-bp}
\end{figure}

\paragraph{What the pipeline saves and what gradient isolation saves.}
Under the same split the parameter and optimizer memory of the two methods is the same, up to the
auxiliary heads, so the pipeline accounts for that part of the saving reported in
Section~\ref{sec:resources} and gradient isolation accounts for the activation part. Within
pipeline parallelism, backpropagation must give up activation memory or utilization, through
recomputation, through the number of micro-batches in flight, or through an uneven split, and SOLO
faces no such choice. Against data parallelism the picture is different. On this node
data-parallel backpropagation is the fastest configuration at every $p$, at $40.3$\,GB per device and
$64$\,GiB of gradient traffic per step, and SOLO under data parallelism runs at $1/\rho$ of its
throughput, with the same activation saving as under a pipeline.

\paragraph{Limitations.}
All measurements are on one node with NVLink and on synthetic tokens, and the model is the one of
configuration (A) rather than the models of Section~\ref{sec:pretraining}. Traffic is computed from
tensor shapes rather than measured on the wire. We did not measure zero-bubble schedules \citep{qi2024zerobubble}, which
split each backward pass into its input-gradient and weight-gradient halves and fill the
bubble with the latter; their ZB-H1 variant keeps the activation memory of 1F1B at about a
third of its bubble, and ZB-H2 removes the bubble at higher memory. On the bubble model with
the fitted $\rho$, ZB-H1 would be about level with SOLO at $M{=}24$ and ahead of it at
$M{=}72$. The
VPP baseline uses the PyTorch implementation, which sends fp32 and reaches $0.80$ to $0.95$ of
its own bubble model, so its throughput is a lower bound on what interleaving can achieve. The
benefit depends on the shape of the model through \eqref{eq:rho}, and
Table~\ref{tab:pipe-shapes} gives $M^{\star}$ for common shapes. A shallow model split into many
stages is the hardest case: at $L{=}24$, $d{=}2048$, $T{=}2048$, $V{=}32$k and $H{=}1$, eight
modules give $\rho=1.49$ and $M^{\star}=14$, against $\rho=1.09$ and $M^{\star}=70$ for the
$96$-block model measured here. Finally, this appendix reports systems costs only. The perplexity
gap of Section~\ref{sec:pretraining} is not included in any of these ratios.

\begin{table}[t]
\centering
\caption{Crossover $M^{\star}$ from \eqref{eq:win} for common model shapes. SOLO is faster than
1F1B when the number of micro-batches per step is below $M^{\star}$. The value does not depend on
the number of stages. Computed, not measured.}
\label{tab:pipe-shapes}
\small
\begin{tabular}{lrrrrrr}
\toprule
parameters & $L$ & $d$ & $T$ & $V$ & $M^{\star}$, $H{=}1$ & $M^{\star}$, $H{=}2$ \\
\midrule
7B  & 32 & 4096 & 4096 & 32{,}000  & 23.7 & 13.7 \\
8B  & 32 & 4096 & 8192 & 128{,}256 & 14.7 & 10.3 \\
70B & 80 & 8192 & 4096 & 32{,}000  & 66.9 & 36.5 \\
70B & 80 & 8192 & 8192 & 128{,}256 & 46.5 & 29.5 \\
\bottomrule
\end{tabular}
\end{table}

\end{document}